\documentclass[journal]{IEEEtran}
\usepackage{amsmath,amssymb,amsfonts}
\usepackage{cases}
\usepackage{color,xcolor}
\usepackage{graphicx}
\usepackage{cite}
\usepackage{subfigure}
\usepackage{multirow}
\usepackage{rotating}
\usepackage{booktabs}
\usepackage{tikz}
\usepackage{acronym}
\usepackage{epstopdf}
\usepackage{pifont} 
\usepackage[demo,
            export]{adjustbox}  
\usepackage{tabularray}
\usepackage{colortbl}
\usepackage[colorlinks,
linkcolor=blue,
anchorcolor=blue,
citecolor=blue
]{hyperref}
\usepackage{amssymb}
\usepackage{array} 
\usepackage{catchfilebetweentags}
\usepackage[ruled, vlined]{algorithm2e}
\newcolumntype{Z}{>{\centering\arraybackslash}p{1.2cm}}
\newcolumntype{W}{>{\centering\arraybackslash}p{1.5cm}}

\usepackage{diagbox}
\usepackage{multicol}
\usepackage{algpseudocode}
\usepackage{algorithmicx}

\newtheorem{Theorem}{Theorem}[section]
\newtheorem{theorem}[Theorem]{Theorem} 

\newtheorem{remark}[Theorem]{Remark}   

\begin{document}

\title{Efficient Personalized Federated PCA \\ with Manifold Optimization for  \\IoT Anomaly Detection}

\author{Xianchao Xiu, Chenyi Huang, Wei Zhang, and Wanquan Liu, \IEEEmembership{Senior Member,~IEEE}
\thanks{This work was supported in part by the National Natural Science Foundation of China under Grant 12371306. (\textit{Corresponding author: Wei Zhang}.)}
\thanks{Xianchao Xiu, Chenyi Huang, and Wei Zhang are with the School of Mechatronic Engineering and Automation, Shanghai University, Shanghai 200444, China (e-mail: xcxiu@shu.edu.cn; huangchenyi@shu.edu.cn; wzhang@shu.edu.cn).}
\thanks{Wanquan Liu is with the School of Intelligent Systems Engineering, Sun Yat-sen University, Guangzhou 510275, China (e-mail: liuwq63@mail.sysu.edu.cn).}
}

\maketitle

\begin{abstract}
Internet of things (IoT) networks face increasing security threats due to their distributed nature and resource constraints. 
Although federated learning  (FL) has gained prominence as a privacy-preserving framework for distributed IoT environments, current federated principal component analysis (PCA) methods lack the integration of personalization and robustness, which are critical for effective anomaly detection. 
To address these limitations, we propose an efficient personalized federated PCA (FedEP) method for anomaly detection in IoT networks. 
The proposed model achieves personalization through introducing local representations with the $\ell_1$-norm for element-wise sparsity, while maintaining robustness via enforcing local models with the $\ell_{2,1}$-norm for row-wise sparsity.
To solve this non-convex problem, we develop a manifold optimization algorithm based on the alternating direction method of multipliers (ADMM) with rigorous theoretical convergence guarantees. 
Experimental results confirm that the proposed FedEP outperforms the state-of-the-art FedPG, achieving excellent F1-scores and accuracy in various IoT security scenarios. 
Our code will be available at \href{https://github.com/xianchaoxiu/FedEP}{https://github.com/xianchaoxiu/FedEP}.
\end{abstract}



\section{Introduction}
\IEEEPARstart{I}{nternet} of things (IoT) has emerged as the foundation of modern wireless communication, bridging the gap between cyber and physical systems  \cite{aouedi2024survey}.
By facilitating data-driven autonomous decision-making, IoT technologies are driving digital transformation in a wide range of domains, such as smart cities \cite{rejeb2022big}, precision agriculture \cite{zhang2025internet}, and next-generation healthcare \cite{bollineni2025iot}.
However, this ubiquitous connectivity makes IoT ecosystems inherently vulnerable to various sophisticated cyber threats \cite{yalli2025systematic, tu2025distributed}.

As a primary line of defense, anomaly detection plays a critical role in building resilient security frameworks in IoT ecosystems \cite{huang2025deep}. 
It identifies diverse types of malicious traffic via comparative  analysis \cite{inuwa2024comparative}, while addressing the evolving challenges of real-time detection of potential security breaches \cite{adhikari2024recent}. 
Traditional anomaly detection frameworks adhere to centralized architectures, which necessitate the backhaul of massive volumes of raw data to cloud data centers for intensive processing \cite{zhou2017security}. 
However, such architectures encounter significant bottlenecks in practical IoT deployments.
Specifically, the computational complexity of advanced detection models often exceeds the stringent resource constraints of edge devices \cite{trilles2024anomaly}. 
Furthermore, the sensitive nature of user-generated data subjects direct data sharing to rigorous legal and ethical scrutiny, notably under frameworks such as the general data protection regulation (GDPR) \cite{kounoudes2020mapping}. 
Thus, there is an urgent need for lightweight, distributed, and privacy-preserving solutions capable of providing robust detection and upholding data sovereignty.

In the last decade, various methodologies have been proposed in the literature. 
Centralized methods leverage complex architectures, such as dual auto-encoder generative adversarial networks (GANs) for industrial inspection, statistical analysis-driven autoencoders \cite{ieracitano2020novel}, and innovative hybrid deep learning frameworks designed for complex IoT network environments \cite{zulfiqar2024deepdetect}, achieving significant detection accuracy. 
However, their reliance on centralized repositories poses significant privacy risks. 
In this regard, federated learning (FL) stands out as a promising distributed architecture that mitigates data silos and enhances privacy by enabling local model training on edge devices and aggregating only model parameters or gradients at a central server \cite{mcmahan2017communication,nguyen2021federated}. 
Although extensive research has applied FL to IoT intrusion detection, including group-based FL methods such as FedGroup tailored for heterogeneous environments \cite{zhang2024privacy}, these deep learning-based FL methods often incur prohibitive computational and communication overhead.
Their massive parameter counts frequently overwhelm resource-constrained edge IoT devices \cite{tang2022computational}, where resource-intensive backpropagation processes further challenge the feasibility of existing computation offloading strategies \cite{han2019federated}. 
Therefore, the limited computational capacity and constrained battery life of sensors necessitate more lightweight reconstructed models to maintain operational efficiency \cite{zhou2024reconstructed}.

Recently, linear models like principal component analysis (PCA) have regained prominence for their efficiency in feature extraction and dimensionality reduction \cite{carter2022fast,xiu2025bi}. 
To preserve privacy, recent research has pivoted toward federated PCA (FedPCA), which enables cooperative subspace estimation without requiring raw data exchange \cite{grammenos2020federated, nguyen2024federated}.
However, existing FedPCA  encounters critical bottlenecks when deployed in heterogeneous IoT environments.
On the one hand, the necessity of personalization is ignored or underestimated \cite{wang2023high}. 
IoT data exhibits inherent non-independent and identically distributed (non-IID) characteristics in diverse sensors and locations \cite{luo2024privacy}, enforcing a monolithic global model leads to a suboptimal fit for local distributions, thereby reducing detection sensitivity at the gateway level \cite{hoang2018pca}. 
On the other hand, robustness to data corruption is often insufficient \cite{chen2021bridging}. 
Classical PCA is sensitive to noise and outliers, thus without robust mechanisms, malicious perturbations or sensor malfunctions in raw telemetry can significantly distort the estimated subspace \cite{rousseeuw2018anomaly,liu2025similarity}.

\begin{figure}[t]
    \centering  
    \includegraphics[width=7cm]{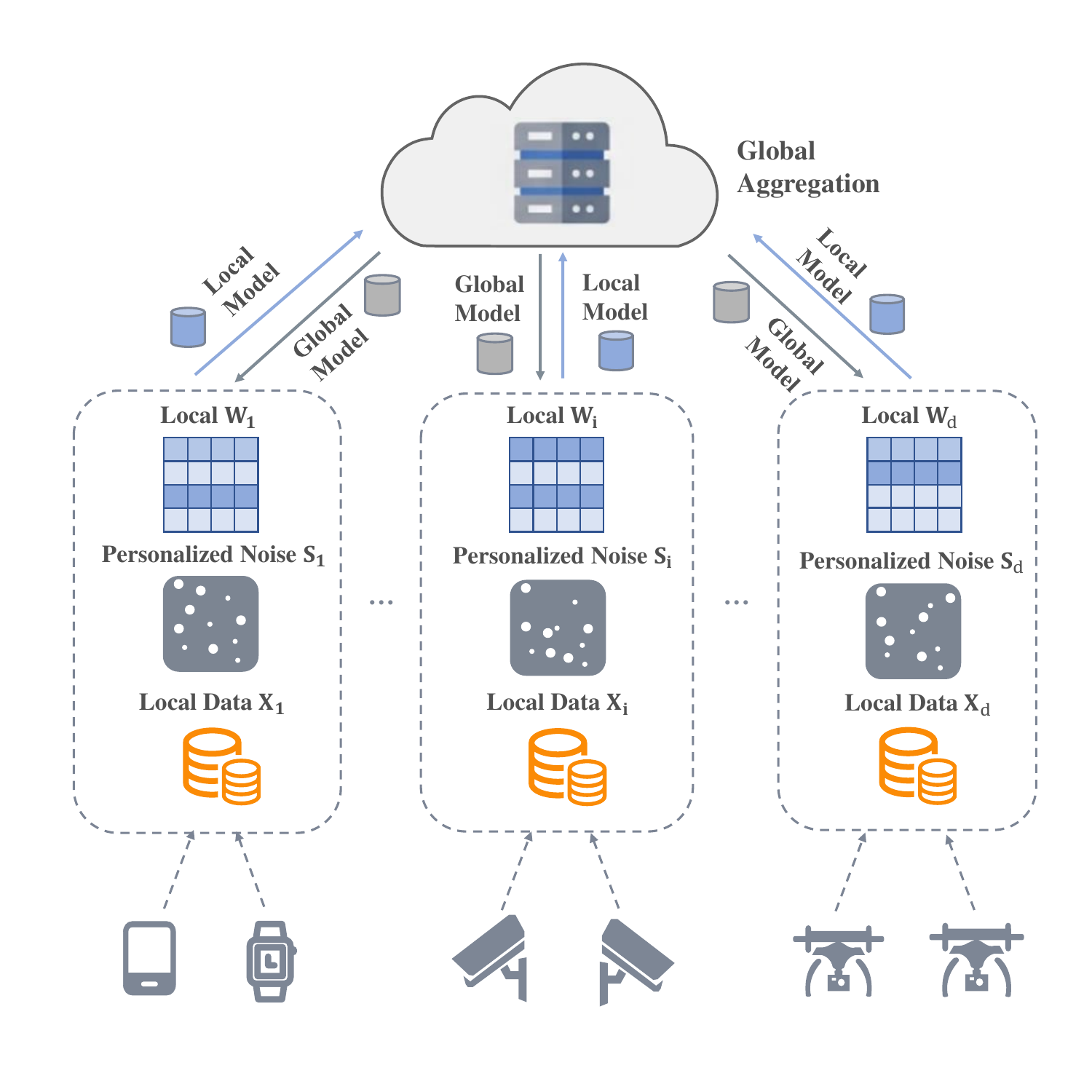}  
    \vspace{-0.6cm}
    \caption{Framework of the proposed FedEP.}
    \label{framework}
\end{figure}

To address these issues, we propose an efficient personalized federated PCA (FedEP) framework for IoT anomaly detection, whose overall structure is depicted in Fig.~\ref{framework}. 
The core idea is to formulate the local data of each client as a superposition of a low-rank component and a sparse component. 
The low-rank component captures normal behavior patterns of the device, incorporating both globally shared and locally personalized features, while the sparse component robustly absorbs noise and isolates potential anomalies. 
Our model uniquely integrates personalization with a bi-sparse regularization scheme, where the $\ell_{2,1}$-norm is utilized to enforce row-wise sparsity for structural robustness \cite{sun2023learning} and the $\ell_1$-norm provides element-wise sparsity to filter fine-grained noise \cite{liu2024towards}.

The main contributions of this paper are as follows.
\begin{itemize}
    \item We propose a novel method for IoT anomaly detection, which allows each local model to capture its specific data distribution while preserving the consistency of the global structure. 
To the best of our knowledge, it is the first framework that integrates  personalization and robustness into the FedPCA framework.
    \item We develop an efficient algorithm based on the alternating direction method of multipliers (ADMM) and manifold optimization, 
whose subproblems admit fast solvers. Furthermore, we provide rigorous convergence analysis.
    \item We evaluate the proposed method on three real-world IoT datasets, and the experimental results illustrate that it achieves superior detection accuracy and lower resource overhead compared to state-of-the-art method.
\end{itemize}

The structure of this paper is outlined as follows. Section \ref{related} reviews related works. Section \ref{method} presents our proposed model and optimization algorithm. Section \ref{experiments} provides experimental results and analysis. Section \ref{conclusions} concludes this paper.

\section{Related Work}\label{related}
This section introduces basic notations, sparse PCA, and FedPCA for IoT anomaly detection.

\subsection{Notations}
Throughout this work, matrices are denoted by uppercase letters, vectors by boldface lowercase letters, and scalars by lowercase italic letters. The set of $m \times n$ real matrices is denoted by $\mathbb{R}^{m \times n}$. For any matrix $X \in \mathbb{R}^{m \times n}$, its $i$-th row and $ij$-th element are represented by $\mathbf{x}^i$ and $x_{ij}$, respectively.
In addition, the Frobenius norm, $\ell_{2,1}$-norm, and $\ell_{1}$-norm of $X$ is defined as $\|X\|_{\mathrm{F}} = \sqrt{\mathrm{Tr}(X^\top X)}$, $\|X\|_{2,1} = \sum_{i=1}^{m} \|\mathbf{x}^i\|_2$, and $\|X\|_{1} = \sum_{i,j} |x_{ij}|$, respectively.
The inner product between two matrices $A$ and $B$ is defined as $\langle A, B \rangle = \mathrm{Tr}(A^\top B)$, where $\mathrm{Tr}(\cdot)$ denotes the trace operator. 
Furthermore, the Stiefel manifold is denoted by $\mathrm{St}(n, m) = \{W \in \mathbb{R}^{n \times m} \mid W^{\top}W = I_m\}$, where $I_m$ is the $m \times m$ identity matrix. 
Additional notations will be introduced where necessary.

\subsection{Sparse PCA}
PCA is a foundational dimensionality reduction method in unsupervised learning, aiming to find an orthogonal projection matrix $W \in \mathbb{R}^{n \times m}$ ($m \ll n$) that minimizes the reconstruction residual. The mathematical model is 
\begin{equation}\label{pca_reconstruction}
    \begin{array}{cl}
         \displaystyle \min_{W}  & \| X - WW^{\top}X \|_{\text{F}}^2 \\
         \text{s.t.}& W^{\top}W = I_m.
    \end{array}
\end{equation}
Although the above PCA is efficient in practice, it generates a dense matrix $W$ that can be sensitive to irrelevant features. To overcome this problem, sparse PCA (SPCA) \cite{zou2006sparse} introduces the $\ell_1$-norm regularization term to encourage sparsity in $W$, which is described as
\begin{equation}\label{spca}
    \begin{array}{cl}
        \displaystyle \min_{W} & \| X - WW^{\top}X \|_\text{F}^2 + \beta \|W\|_{2,1} \\
        \text{s.t.}  & W^{\top}W = I_m,
    \end{array}
\end{equation}
where $\beta>0$ is the parameter to tune the sparsity of $W$. Another impressive work is robust sparse PCA (RSPCA) \cite{rousseeuw2018anomaly}, which addresses the challenge of gross corruptions by decomposing the data into a low-rank component and a sparse error matrix. The model is given by
\begin{equation}\label{rpca}
    \begin{array}{cl}
        \displaystyle  \min_{W, S} & \| (I - WW^{\top})(X - S) \|_\text{F}^2 + \alpha \|S\|_1 \\
        \text{s.t.}  & W^{\top}W = I_m,
    \end{array}
\end{equation}
where $S$ captures the outlier-induced noise, allowing the subspace $W$ to accurately represent the underlying data structure, and $\alpha>0$ is the parameter to balance the noise level of  $S$.

\subsection{FedPCA for IoT Anomaly Detection}
In the field of IoT, PCA-based anomaly detection identifies malicious traffic by projecting high-dimensional telemetry into a low-rank subspace. A sample $\mathbf{x}^i$ is marked as an anomaly if its reconstruction error 
\begin{equation}
S(\mathbf{x}^i) = \| (I - WW^{\top})\mathbf{x}^i\|_2^2
\end{equation}
exceeds a threshold \cite{trilles2024anomaly}. 
However, traditional PCA struggles to handle  malicious outliers and the privacy risks associated with centralized data aggregation.

To this end, FedPCA \cite{nguyen2024federated} was developed, with the FedPCA on Grassmann (FedPG) framework being a representative method that seeks a global consensus among $d$ gateways. Suppose that $X_i\in \mathbb{R}^{m \times n}$ is the local data in the $i$-th local client, and $W_i\in \mathbb{R}^{n \times m}$ is the corresponding projection matrix, then FedPG is given by
\begin{equation}\label{fed_pca_consensus}
    \begin{array}{cl}
        \displaystyle \min_{\{W_i\}, V} & \sum_{i=1}^d \| X_i - W_iW_{i}^{\top}X_i \|_\mathrm{F}^2 \\
        \text{s.t.} & W_i = V, \quad W_i^{\top} W_{i}=I_m, \quad \forall i \in [d],
    \end{array}
\end{equation}
where $[d]=\{1, 2, \cdots, d\}$, and $V$ is the global consensus variable. 
Compared to the aforementioned PCA in \eqref{pca_reconstruction}, FedPG demonstrates outstanding performance in IoT anomaly detection.
Here 

\section{The Proposed Method}\label{method}
\subsection{Problem Formulation}
In this paper, we propose an efficient personalized
federated PCA (FedEP)  by relaxing the consensus into a personalized objective with bi-sparse regularization, simultaneously enhancing local adaptability and structural robustness. The objective is formulated as
\begin{equation}\label{eq:proposed_method}
    \begin{array}{cl}
        \displaystyle \min_{\{W_i\}, \{S_i\}, V} \quad & \sum_{i=1}^d (\| (I - W_i W_{i}^{\top})(X_i - S_i) \|_\textrm{F}^2  \\
        & + ~\alpha \|S_i\|_1 + \beta \|W_i\|_{2,1})\\
        \textrm{s.t.} & W_i = V, \quad W_i^{\top} W_{i}=I, \quad \forall i \in [d].
    \end{array}
\end{equation}

Compared with existing frameworks like FedPG \cite{nguyen2024federated},  our FedEP offers two distinct advantages:
\begin{itemize}
    \item (\textit{Adaptive Personalization}) The sparse component $S_i$ absorbs local-specific noise and outliers, allowing the model to adapt to the non-IID nature of IoT data. 
   \item (\textit{Structural Robustness}) The $\ell_{2,1}$-norm on $W_i$ enforces row-wise sparsity, making it resilient to corrupted features while reducing the computation.
\end{itemize}

\subsection{Optimization Algorithm}
For the sake of description, denote
\begin{equation}
p(S_i,W_i)=\alpha \|S_i\|_1 + \beta \|W_i\|_{2,1}.
\end{equation}
To facilitate optimization, we introduce the auxiliary variable $U_i$, which represents the target of projection of the denoised and purified local data.
Then, problem \eqref{eq:proposed_method} is transformed into 
\begin{equation}\label{eq:reformulated_prob}
    \begin{array}{cl}
        \displaystyle \min_{\{W_i\}, \{S_i\}, \{U_i\}, V }  & \sum_{i=1}^d (\| (I - W_i W_{i}^{\top})U_i \|_\textrm{F}^2 + p(S_i,W_i)\\
        \textrm{s.t.} & X_i - S_i = U_i, \quad \forall i \in [d], \\
        & W_i - V = 0, \quad \forall i \in [d],\\
        & W_i^{\top} W_{i}=I, \quad \forall i \in [d].
    \end{array}
\end{equation}

According to the ADMM scheme, the global augmented Lagrangian function $\mathcal{L}$ can be decomposed into the sum of $d$ local functions as
\begin{equation}\label{lag}
    \begin{array}{cl}
		&\mathcal{L} (\{W_i\}, \{S_i\}, \{U_i\}, V; \{\Lambda_i\}, \{\Pi_i\}) \\
        &= \sum_{i=1}^d \mathcal{L}_i (W_i, S_i, U_i, V; \Lambda_i, \Pi_i),
    \end{array}
\end{equation}
where $\mathcal{L}_i$ is defined as
\begin{equation}\label{eq:Lagrangian}
    \begin{aligned}
        & \mathcal{L}_i (W_i, S_i, U_i, V; \Lambda_i, \Pi_i) \\ 
        & = \| (I - W_i W_{i}^{\top})U_i \|_\textrm{F}^2 + p(S_i,W_i) \\
        & + \langle \Lambda_i, X_i - S_i - U_i \rangle + \frac{\mu }{2} \|X_i - S_i - U_i\|_\textrm{F}^2 \\
        & + \langle \Pi_i, W_i - V \rangle + \frac{\nu}{2} \|W_i - V\|_\textrm{F}^2.
    \end{aligned}
\end{equation}
Here, $\Lambda_i$ and $\Pi_i$ are the Lagrange multipliers, and $\mu,\nu > 0$ are the penalty parameters. This decentralized structure ensures that each client $i$ can update its local variables $W_i$ independently, while global synchronization is only required for the consensus variable $V$, thereby preserving data privacy and reducing communication overhead.

%

\begin{algorithm*}[!th]
    \caption{FedEP via ADMM with Manifold Optimization}\label{admm}
    \SetAlgoLined
    \SetKwInput{Input}{Input}
    \SetKwInput{Initialize}{Initialize}
    \Input{Data $\{X_i\}$, parameters $\alpha, \beta, \mu, \nu,\gamma \in (0,1)$, threshold $\varepsilon$, dimension $m$,  step size $t > 0$,  and $K_{\max}, T_{\max}$}
    \Initialize{Set $k=0$, $W_i^0 \in \mathrm{St}(n,m)$, $S_i^0, U_i^0, V^0$, $\Lambda_i^0, \Pi_i^0$, $i \in [d]$}
    
    \While{$k < K_{\max}$}{
        \BlankLine
        \tcp*[l]{Local Update Stage}
        \For{every $i \in [d]$}{
            \underline{\textit{Update $W_i^{k+1}$}}\\
            Set $j=0$, initialize $W_i^{(0)}$ as $W_i^k$\\
            \While{$j < T_{\max}$}{
                Compute Euclidean gradient $\nabla H(W_i^{(j)})$ via \eqref{eq:grad_W}\\
                Obtain descent direction $D_i^{(j)}$ by solving $Q(\Lambda_i) = 0$ via the SSN method\\
                Set $\alpha = 1$\\
                \While{$F(\mathrm{Retr}_{W_i^{(j)}}(\alpha D_i^{(j)})) > F(W_i^{(j)}) - \frac{\alpha}{2t} \|D_i^{(j)}\|_\mathrm{F}^2$}{
                    $\alpha = \gamma \alpha$
                }
                Update  $W_i^{(j+1)} = \mathrm{Retr}_{W_i^{(j)}}(\alpha D_i^{(j)})$\\
                $j = j + 1$
            }
            Set $W_i^{k+1} = W_i^{(j)}$\\
            
            \BlankLine
            \underline{\textit{Update local variables}}\\
            Update $S_i^{k+1}$ via \eqref{sol:S}\\
            Update $U_i^{k+1}$ via  \eqref{sol:U}\\
            Update $\Lambda_i^{k+1}, \Pi_i^{k+1}$ via  \eqref{sol:Lambda}
        }
        
        \BlankLine
        \tcp*[l]{Global Aggregation Stage }
        Update global variable $V^{k+1}$ via \eqref{sol:V} \\
        $k = k + 1$
    }
    \textbf{Output:} Optimized global variable $V$
\end{algorithm*}

\subsubsection{Update $W_i$}
The subproblem with respect to $W_i$ can be formulated as
\begin{equation}\label{prob:W}
    \begin{array}{cl}
        \displaystyle \min_{W_i \in \mathrm{St}(n,m)}  & \| (I_n - W_i W_i^{\top}) U_i^k \|_{\mathrm{F}}^2 + \beta \|W_i\|_{2,1} \\
        & + ~\frac{\nu}{2} \|W_i - Z_i^k\|_{\mathrm{F}}^2,
    \end{array}
\end{equation}
where $Z_i^k = V^k - \Pi_i^k/\nu$.
The objective function in \eqref{prob:W} consists of two smooth components and a non-smooth regularization term, and subject to the non-convex Stiefel manifold, which makes it computationally challenging.
For simplicity, we reformulate problem \eqref{prob:W} as 
\begin{equation}
    \min_{W_i \in \mathrm{St}(n,m)} \quad F(W_i)= H(W_i) + g(W_i),
\end{equation}
where 
\begin{equation}
\begin{aligned}
H(W_i) &= \|U_i^k - W_i W_i^\top U_i^k\|_\mathrm{F}^2 + \frac{\nu}{2} \|W_i - Z_i^k\|_\mathrm{F}^2, \\
g(W_i) &= \beta \|W_i\|_{2,1}.
\end{aligned}
\end{equation}
Obviously, $H(W_i)$ is smooth, $g(W_i)$ is non-smooth.
The optimization process begins by deriving the Euclidean gradient of  $H(W_i)$. By expanding the Frobenius norm terms and exploiting the orthogonality property $W_i^\top W_i = I_m$, the gradient at the current iterate $W_i^k$ is obtained as
\begin{equation}
    \nabla H(W_i^k) = -2 U_i^k (U_i^k)^\top W_i^k + \nu (W_i^k - Z_i^k).
    \label{eq:grad_W}
\end{equation}
To deal with $g(W_i)$ while strictly satisfying the geometry of the Stiefel manifold, the alternating manifold proximal gradient method \cite{chen2020alternating} is considered. The direction $D_i$ is determined by solving 
\begin{equation}
    \begin{array}{cl}
        \displaystyle \min_{D_i}  & \langle \nabla H(W_i^k), D_i \rangle + \frac{1}{2t} \|D_i\|_\mathrm{F}^2 + \beta \|W_i^k + D_i\|_{2,1} \\
        \text{s.t.}   & D^\top W_i^k + (W_i^k)^\top D = 0,
    \end{array}
    \label{eq:subproblem}
\end{equation}
where $t > 0$ is the step size. Once the optimal direction $D_i^*$ is computed, the next iterate is obtained via a retraction operation $W_i^{k+1} = \mathrm{Retr}_{W_i^k}(\alpha D_i^*)$, where $\alpha$ is a step size determined by a backtracking line search to ensure sufficient descent.

Let $\Lambda_i \in \mathbb{R}^{m \times m}$ be the symmetric Lagrange multiplier associated with the tangent space constraint. The first-order optimality condition leads to a closed-form expression for $D_i$ in terms of $\Lambda_i$ as
\begin{equation}
    D_i(\Lambda_i) = \mathrm{prox}_{2,1}(B(\Lambda_i), t\beta) - W_i^k,
\end{equation}
where $B(\Lambda_i) = W_i^k - t(\nabla H(W_i^k) - W_i^k \Lambda_i)$, and $\mathrm{prox}_{2,1}$ is the proximal operator associated with the $\ell_{2,1}$-norm \cite{zhu2025sparse}. By substituting $D_i(\Lambda_i)$ back into the tangent space constraint, problem  \ref{eq:subproblem} reduces to finding the root of the following equation system
\begin{equation}
    Q(\Lambda_i) = D_i(\Lambda_i)^\top W_i^k + (W_i^k)^\top D_i(\Lambda_i) = 0.
\end{equation}
Inspired by \cite{xiao2018regularized}, the regularized semi-smooth Newton (SSN) method is applied to solve $Q(\Lambda_i) = 0$, which guarantees fast convergence for finding the optimal descent direction.

\subsubsection{Update \texorpdfstring{$S_i$}{Si}}
Through simple algebraic operations, the $S_i$-subproblem can be easily written as
\begin{equation}\label{prob:S}
    \min_{S_i} \quad  \alpha \|S_i\|_1 + \frac{\mu}{2} \|S_i - M_i^k\|_\textrm{F}^2,
\end{equation}
where $M_i^k = X_i - U_i^k + \Lambda_i^k/\mu$. According to \cite{donoho1995noising}, the closed-form solution is characterized by the following soft-thresholding operator
\begin{equation}\label{sol:S}
    S_i^{k+1} = \textrm{sign}(M_i^k) \odot \max(|M_i^k| - \alpha/\mu, 0),
\end{equation}
where $\odot$ means the element-wise product.

\subsubsection{Update \texorpdfstring{$U_i$}{Ui}}
Once $W_i^{k+1}$ and $ S_i^{k+1}$ have been obtained, the $U_i$-subproblem can be solved by 

\begin{equation}\label{prob:U}
    \begin{array}{cl}
        \displaystyle \min_{U_i}  & \| (I - W_i^{k+1}(W_i^{k+1})^{\top}) U_i \|_{\mathrm{F}}^2 \\
        & +~ \frac{\mu}{2} \| U_i - (X_i - S_i^{k+1} + \Lambda_i^k/\mu) \|_{\mathrm{F}}^2
    \end{array}
\end{equation}
Taking the derivative with respect to $U_i$, it is easy to obtain

\begin{equation}
    \begin{aligned}
        & 2 (I - W_i^{k+1}(W_{i}^{k+1})^{\top} ) U_i& \\ & +\mu  (U_i - (X_i - S_i^{k+1} + \Lambda_i^k/\mu) ) = 0.
    \end{aligned}
\end{equation}
Let $Y_i^k = X_i - S_i^{k+1} + \Lambda_i^k/\mu$, to speed up the calculation of matrix inversion, the famous Sherman-Morrison-Woodbury formula can be used \cite{hager1989updating}, thus the solution is
\begin{equation}\label{sol:U}
    U_i^{k+1} = ( \frac{\mu}{\mu + 2} I + \frac{2}{\mu + 2} W_i^{k+1} (W_i^{k+1})^{\top} ) Y_i^k.
\end{equation}

\subsubsection{Update \texorpdfstring{$V$}{V}}
Fixing other variables, the global $V$-subproblem can be simplified to
\begin{equation}
    \min_{V} \quad \sum_{i=1}^d \frac{\nu}{2} \|W_i^{k+1} - V + \Pi_i^k/\nu \|_\textrm{F}^2,
\end{equation}
which has the following closed-form solution
\begin{equation}
    V^{k+1} = \frac{1}{d} \sum_{i=1}^d ( W_i^{k+1}  + \Pi_i^k/\nu).
\end{equation}
According to \cite{boyd2011distributed}, it can be further reduced into
\begin{equation}\label{sol:V}
    V^{k+1} = \frac{1}{d} \sum_{i=1}^d  W_i^{k+1}.
\end{equation}

\subsubsection{Update \texorpdfstring{$\Lambda_i$ and $\Pi_i$}{Lambda-i and Pi-i}}
Finally, the Lagrange multipliers are calculated as 
\begin{equation}\label{sol:Lambda}
    \begin{aligned}
	\Lambda_i^{k+1} &=  \Lambda_i^k - \mu (X_i^{k+1} - S_i^{k+1} - U_i^{k+1}), \\
	\Pi_i^{k+1}     &=  \Pi_i^k - \nu (W_i^{k+1} - V^{k+1}).
    \end{aligned}
\end{equation}

Therefore, the overall iterative optimization scheme is summarized in Algorithm \ref{admm}.

\subsection{Convergence Analysis}    

Let $(\{W_i^k\}, \{S_i^k\}, \{U_i^k\}, V^k; \{\Lambda_i^k\}, \{\Pi_i^k\})$ be the generated sequence by Algorithm \ref{admm}. Then we establish the convergence result in the following theorem.
\begin{theorem}
\label{converge}
The augmented Lagrangian function sequence $\{\mathcal{L} (\{W_i^k\}, \{S_i^k\}, \{U_i^k\}, V^k; \{\Lambda_i^k\}, \{\Pi_i^k\})\}$ is nonincreasing.
\end{theorem}

\begin{IEEEproof}
Following a similar line as \cite{chen2020alternating}, the corresponding objectives of \eqref{prob:W} are nonincreasing. Thus, it is concluded that 
\begin{equation}
    \label{nonincreasing-W}
    \begin{array}{cl}
        & \mathcal{L}_i (W_i^{k+1}, S_i^k, U_i^k, V^k; \Lambda_i^k, \Pi_i^k) \\ 
        & \leq \mathcal{L}_i (W_i^k, S_i^k, U_i^k, V^k; \Lambda_i^k, \Pi_i^k), \quad \forall i \in [d].
    \end{array}
\end{equation}

Since the $S_i, U_i, V$-subproblems are all convex, the minimization steps satisfy the following inequalities
\begin{equation}
    \label{nonincreasing-V}
    \begin{aligned}
        &\mathcal{L}(W_i^{k+1}, S_i^{k+1}, U_i^{k+1}, V^{k+1}; \Lambda_i^k, \Pi_i^k) \\
        &\leq \mathcal{L} (W_i^{k+1}, S_i^{k+1}, U_i^{k+1}, V^k; \Lambda_i^k, \Pi_i^k)\\
        &\leq \mathcal{L}_i (W_i^{k+1}, S_i^{k+1}, U_i^k, V^k; \Lambda_i^k, \Pi_i^k)\\
        &\leq \mathcal{L}_i (W_i^{k+1}, S_i^k, U_i^k, V^k; \Lambda_i^k, \Pi_i^k), \quad \forall i \in [d].
    \end{aligned}
\end{equation}

Regarding the $\Lambda_i, \Pi_i$-subproblems, the following statements hold
\begin{equation}
    \label{dual-pi}
    \begin{aligned}
      &  \mathcal{L}(W_i^{k+1}, S_i^{k+1}, U_i^{k+1}, V^{k+1}; \Lambda_i^{k+1}, \Pi_i^{k+1}) \\
      &  \leq \mathcal{L} (W_i^{k+1}, S_i^{k+1}, U_i^{k+1}, V^{k+1}; \Lambda_i^{k+1}, \Pi_i^k)\\
             & \leq \mathcal{L} (W_i^{k+1}, S_i^{k+1}, U_i^{k+1}, V^{k+1}; \Lambda_i^k, \Pi_i^k), \quad \forall i \in [d].
    \end{aligned}
\end{equation}

Combining the above inequalities \eqref{nonincreasing-W}, \eqref{nonincreasing-V}, and \eqref{dual-pi}, it is easy to see that
\begin{equation}
    \label{final-nonincreasing}
    \begin{aligned}
        & \mathcal{L} (\{W_i^{k+1}\}, \{S_i^{k+1}\}, \{U_i^{k+1}\}, V^{k+1}; \{\Lambda_i^{k+1}\}, \{\Pi_i^{k+1}\}) \\
        & \leq \mathcal{L} (\{W_i^k\}, \{S_i^k\}, \{U_i^k\}, V^k; \{\Lambda_i^k\}, \{\Pi_i^k\}),
    \end{aligned}
\end{equation}
which shows that the generated augmented Lagrangian function $\{\mathcal{L} (\{W_i^k\}, \{S_i^k\}, \{U_i^k\}, V^k; \{\Lambda_i^k\}, \{\Pi_i^k\})\}$ is nonincreasing, and thus the proof is completed.
\end{IEEEproof}

\begin{remark}
Although \cite{nguyen2024federated} provides the convergence to stationary points, its objective function is smooth, which is clearly not satisfied for our model \eqref{eq:reformulated_prob}. Note that \cite{li2025riemannian} introduces the Riemannian ADMM with convergence guarantees, but its assumptions still do not apply to Algorithm \ref{admm}.
\end{remark}


\section{Experiments}\label{experiments}

To demonstrate the effectiveness, this section compares the proposed FedEP with the state-of-the-art FedPG \cite{nguyen2024federated} in an intrusion detection system (IDS) deployment.
The experiments are conducted on a server equipped with an Intel Ultra 9 Processor 285K, Ubuntu 22.04.4 LTS, 64GB RAM, and NVIDIA RTX 5090 GPU.

\subsection{Experimental Settings}

\begin{table}[t]
    \centering
    \caption{Statistics of the selected IoT datasets.} \label{dataset}
    \vspace{-1mm}
    \renewcommand\arraystretch{1.3}
    \setlength{\tabcolsep}{2pt} 
    \begin{tabular}{|c|c|c|c|c|}
        \hline
        Datasets & Features & Training Samples & Testing Samples & Classes \\
        \hline
        \hline
        TON-IoT & 49 & 114,956 & 66,557 & 10 \\
        \hline
        UNSW-NB15 & 39 & 65,000 & 65,332 & 10 \\
        \hline
        NSL-KDD & 34 & 125,973 & 22,544 & 5 \\
        \hline
    \end{tabular}
\end{table}

\subsubsection{Dataset Description}

The proposed FedEP is evaluated on  three real-world network intrusion detection datasets: TON-IoT \cite{booij2021ton_iot}, UNSW-NB15 \cite{moustafa2015unsw}, and NSL-KDD \cite{tavallaee2009detailed}. Detailed statistics are listed in Table \ref{dataset}.
 
\begin{itemize}
    \item \textbf{TON-IoT}\footnote{https://research.unsw.edu.au/projects/toniot-datasets}: It contains 49 features covering DDoS, DoS, and scanning attacks. After post-processing, it includes 114,956 training samples and 66,557 testing samples. 
    \item \textbf{UNSW-NB15}\footnote{https://research.unsw.edu.au/projects/unsw-nb15-dataset}: It is developed by ACCS, consisting of 39 features with hybrid real normal and synthetic modern attack scenarios. The training and testing datasets include 65,000 samples and 65,332 samples, respectively.
    \item \textbf{NSL-KDD}\footnote{https://www.unb.ca/cic/datasets/nsl.html}: It comprises 34 features across five attack categories (DoS, Probe, R2L, U2R, and Normal), with 125,973 training samples and 22,544 testing samples.
\end{itemize}

\begin{figure*}[t] 
	\centering
	
	\subfigure[TON-IoT]{
		\includegraphics[width=0.3\textwidth]{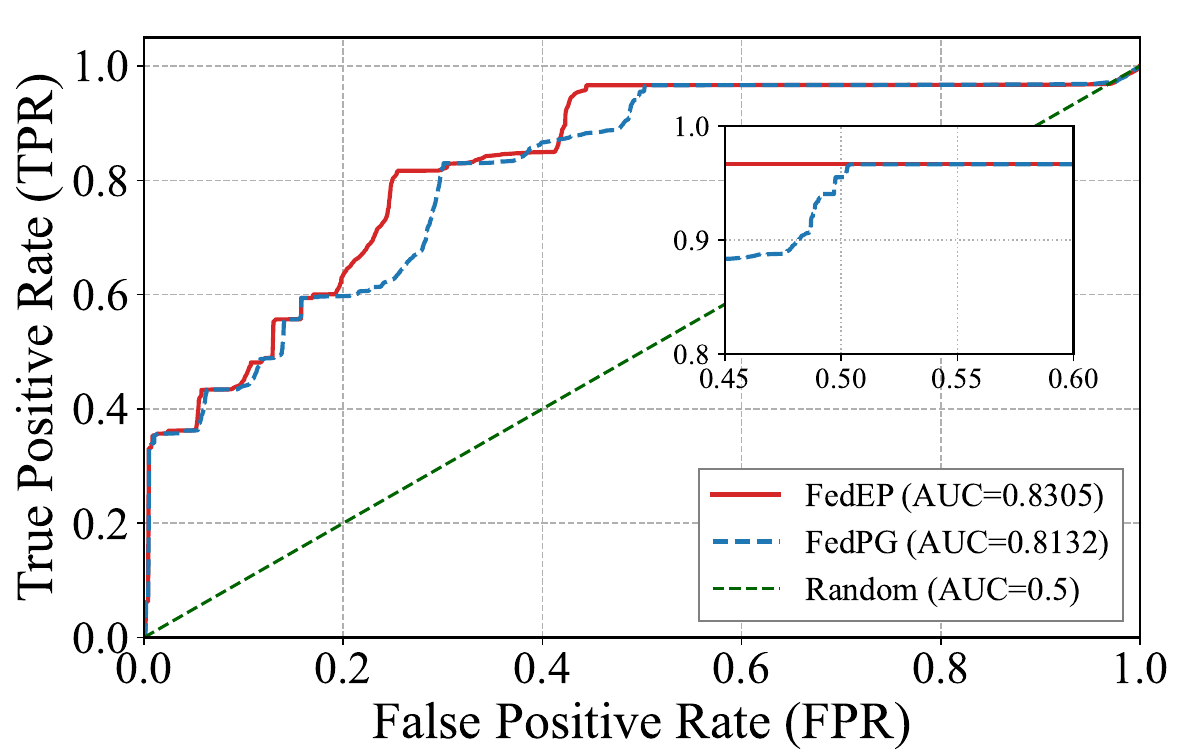}
		\label{fig:ton_roc}
	}\hfill 
	\subfigure[UNSW-NB15]{
		\includegraphics[width=0.3\textwidth]{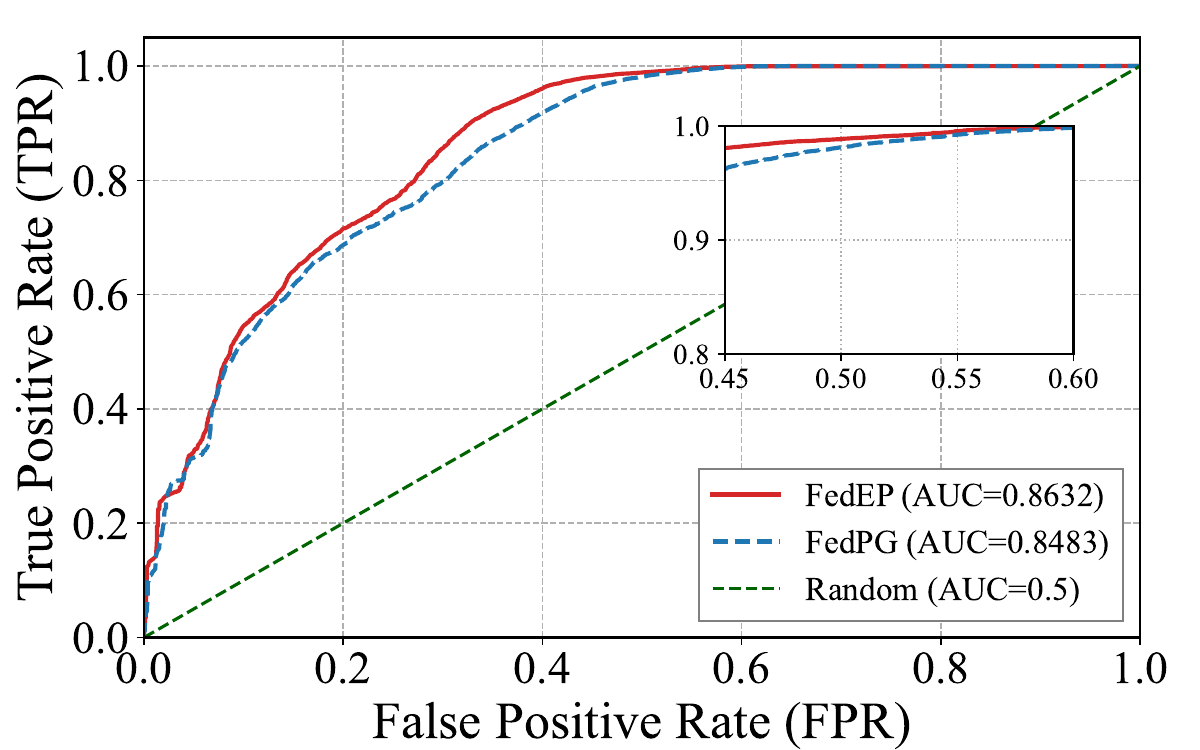}
		\label{fig:unsw_roc}
	}\hfill
	\subfigure[NSL-KDD]{
		\includegraphics[width=0.3\textwidth]{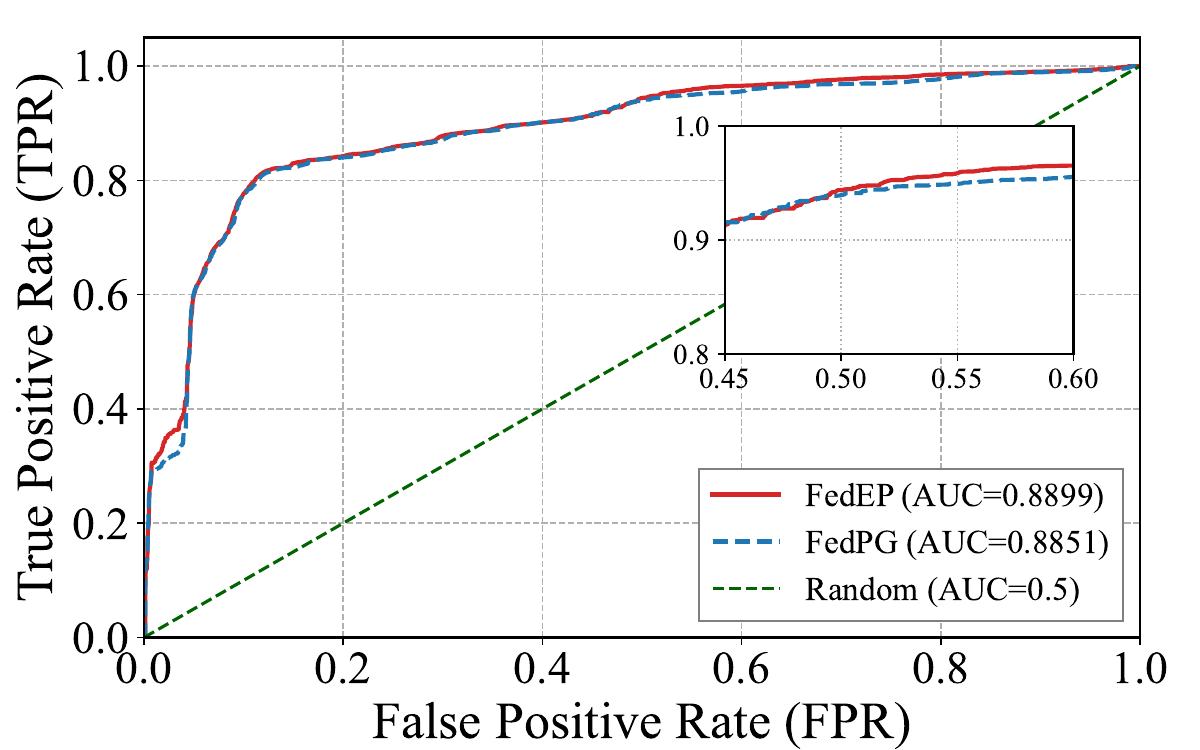}
		\label{fig:kdd_roc}
	}
    
	\vspace{-2mm}
	\caption{ROC curves on  (a) TON-IoT, (b) UNSW-NB15, (c) NSL-KDD.}
	\label{fig:roc_results}
\end{figure*}

\begin{figure*}[t]
	\centering
	\subfigure[TON-IoT]{
		\includegraphics[width=0.3\textwidth]{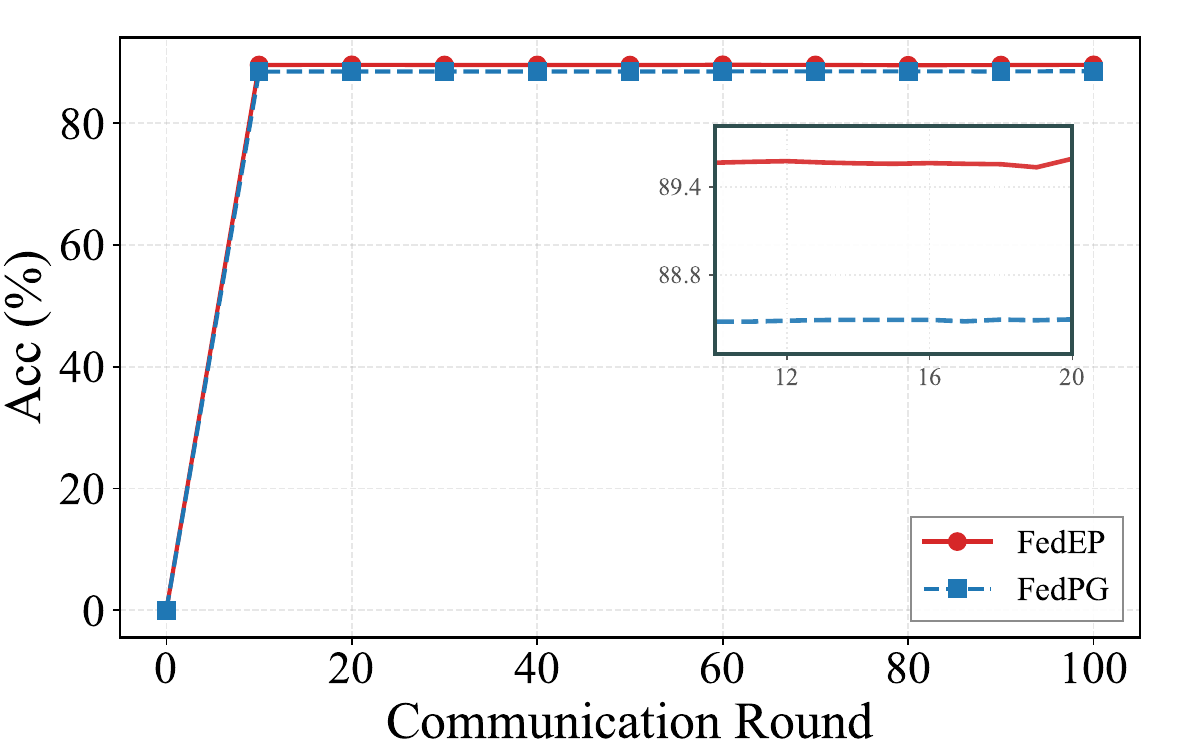}
		\label{fig:ton_acc}
	}\hfill 
	\subfigure[UNSW-NB15]{
		\includegraphics[width=0.3\textwidth]{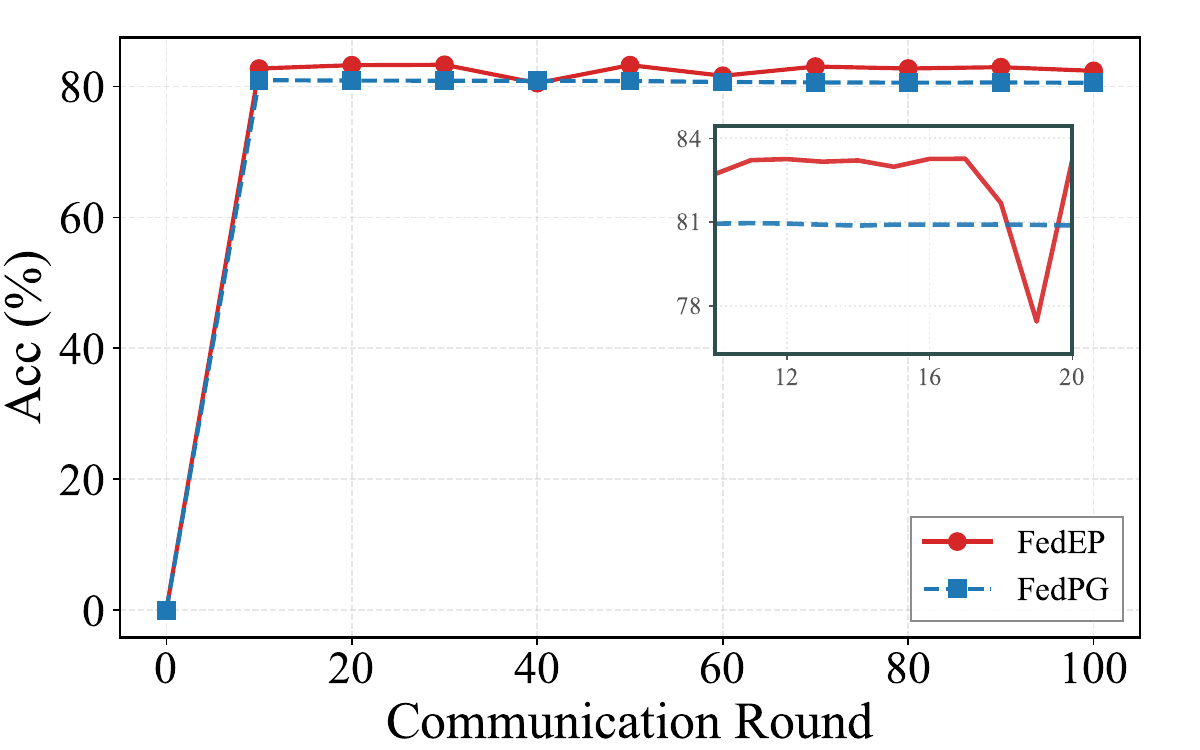}
		\label{fig:unsw_acc}
	}\hfill
	\subfigure[NSL-KDD]{
		\includegraphics[width=0.3\textwidth]{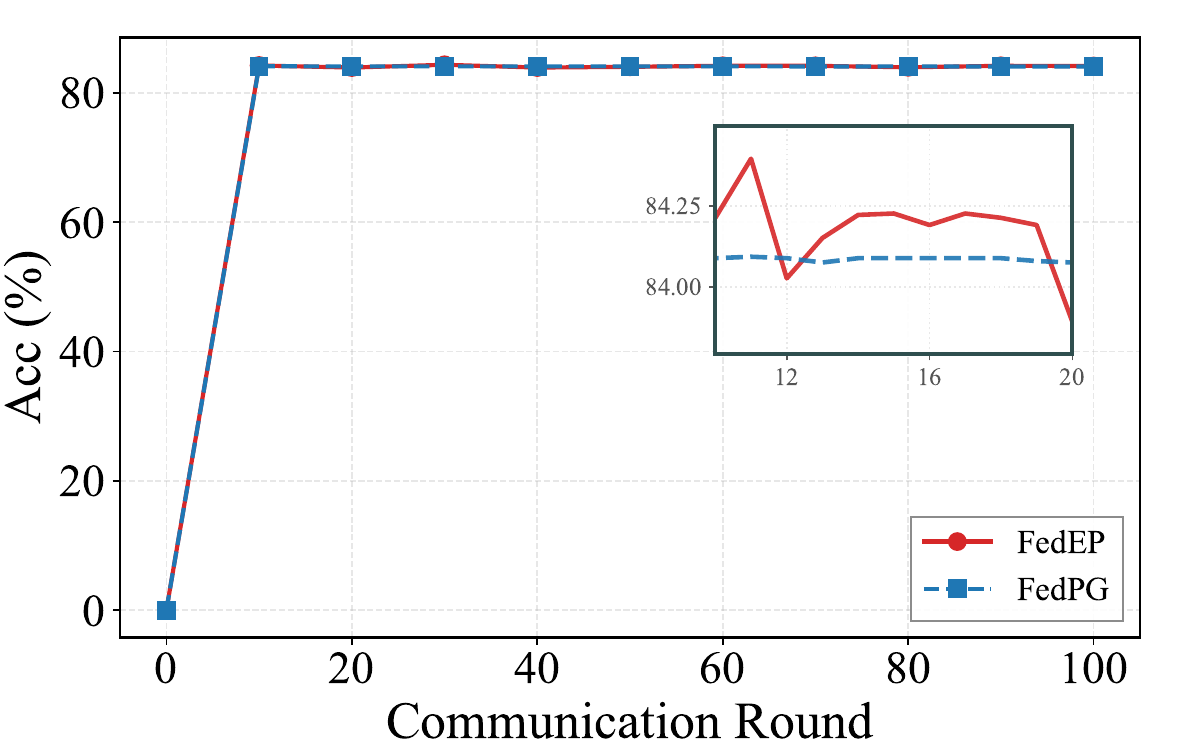}
		\label{fig:kdd_acc}
	}
	\vspace{-2mm}
	\caption{Accuracy  over communication rounds on  (a) TON-IoT, (b) UNSW-NB15, (c) NSL-KDD.}
	\label{fig:overall_results}
\end{figure*}

\subsubsection{Evaluation Metrics}

To evaluate the performance of compared methods, five standard metrics are utilized. Let $\textrm{TP}$, $\textrm{TN}$, $\textrm{FP}$, and $\textrm{FN}$ represent true positives, true negatives, false positives, and false negatives, respectively.

\begin{itemize}
    \item Accuracy ($\textrm{Acc}$): The overall proportion of correctly classified records, defined as
    \begin{equation}
        \textrm{Acc} = \frac{\textrm{TP} + \textrm{TN}}{\textrm{TP} + \textrm{TN} + \textrm{FP} + \textrm{FN}}.
    \end{equation}

    \item Precision ($\textrm{Pre}$): The ratio of correctly identified attacks to all predicted attack records, defined as
    \begin{equation}
        \textrm{Pre} = \frac{\textrm{TP}}{\textrm{TP} + \textrm{FP}}.
    \end{equation}

    \item Recall ($\textrm{Rec}$): The ratio of correctly detected attacks to all actual attack records, defined as
    \begin{equation}
        \textrm{Rec} = \frac{\textrm{TP}}{\textrm{TP} + \textrm{FN}}.
    \end{equation}

    \item False Negative Rate ($\textrm{FNR}$): The fraction of actual anomalies that are incorrectly classified as normal, defined as
    \begin{equation}
        \textrm{FNR} = \frac{\textrm{FN}}{\textrm{TP} + \textrm{FN}}.
    \end{equation}

    \item F1-score ($\textrm{F1-score}$): The harmonic mean of precision and recall, providing a balanced measure of model performance, defined as
    \begin{equation}
        \textrm{F1-score}  = \frac{2\textrm{TP}}{2\textrm{TP} + \textrm{FP} + \textrm{FN}}.
    \end{equation}
\end{itemize}
We would like to point out that larger Acc $(\uparrow)$, Pre $(\uparrow)$, Rec $(\uparrow)$, F1-score $(\uparrow)$ values, and smaller FNR $(\downarrow)$ values, indicate better anomaly detection results.

\subsubsection{Implementation Issues}
In the IDS deployment, local IoT devices connect to a gateway to collect data for anomaly detection. To reflect the diverse client traffic, the training samples for each dataset are divided into $20$ non-IID subsets based on representative features, dst$\_$bytesfor TON-IoT, UNSW-NB15, and NSL-KDD. The features are normalized by $z$-scores, and the hyperparameters are optimized by grid search.

\subsection{Performance Comparison}


\begin{table}[t]
    \centering
    \renewcommand\arraystretch{1.3}
    \caption{Detection performance on TON-IoT, UNSW-NB15, and NSL-KDD.}\label{tab2}
    
    \setlength{\tabcolsep}{3pt} 
    
    \begin{tabular}{|c|c|c|c|c|c|c|}
        \hline
        \multirow{2}{*}{\diagbox[width=1.5cm]{Metrics}{Datasets}} & \multicolumn{2}{c|}{TON-IoT} & \multicolumn{2}{c|}{UNSW-NB15} & \multicolumn{2}{c|}{NSL-KDD} \\
        \cline{2-7}
        & FedPG & FedEP & FedPG & FedEP & FedPG & FedEP \\
        \hline
        \hline
        Acc  ($\uparrow$)  & 88.89\% & \textbf{90.48\%} & 80.93\% & \textbf{83.31\%} & 84.09\% & \textbf{84.24\%} \\
        \hline
        Pre  ($\uparrow$)  & 90.84\% & \textbf{92.48\%} & 81.58\% & \textbf{82.17\%} & \textbf{89.78\%} & 89.66\% \\
        \hline
        Recall ($\uparrow$) & \textbf{96.67\%} & 96.65\% & 93.66\% & \textbf{97.00\%} & 81.30\% & \textbf{81.75\%} \\
        \hline
        FNR ($\downarrow$)   & \textbf{3.33\%}  & 3.35\%  & 6.34\%  & \textbf{3.00\%} & 18.70\% & \textbf{18.25\%} \\
        \hline
        F1-score  ($\uparrow$)   & 93.66\% & \textbf{94.52\%} & 87.20\% & \textbf{88.97\%} & 85.33\% & \textbf{85.52\%} \\
        \hline
    \end{tabular}
\end{table}

Unlike FedPG, which necessitates retrieving all records to a centralized server for global variable learning, FedEP preserves data locality through its personalized sparse framework. 
As summarized in Table \ref{tab2}, FedEP obtains comparable performance for all three datasets, particularly in the core metrics of accuracy and F1 score.
Specifically, on TON-IoT, FedEP achieves an accuracy of 90.48\% and an F1-score of 94.52\%, representing a steady improvement over FedPG. 
The advantage of the proposed method becomes more pronounced on the highly challenging UNSW-NB15, where FedEP yields a substantial 2.38\% gain in accuracy and a 1.77\% boost in F1-score. 
Similar results can also be found on NSL-KDD.
The performance improvement in various intrusion detection scenarios highlights the effectiveness of integrating personalized sparse regularization, which can capture heterogeneous local patterns and maintain high detection accuracy.

Furthermore,  Fig. \ref{fig:roc_results} shows the receiver operating characteristic (ROC) curves of compared methods. It is concluded that FedEP achieves higher area under the curve (AUC) values across all datasets, i.e., 0.8305 on TON-IoT, 0.8632 on UNSW-NB15, and 0.8899 on NSL-KDD, compared to  0.8132, 0.8483, and 0.8851 of FedPG, respectively. This indicates that FedEP has a better trade-off between true positive rate and false positive rate under different network environments. 

In addition, Fig. \ref{fig:overall_results} illustrates the evolution of model accuracy across global communication rounds.
It is found that on all three datasets, FedEP achieves its peak accuracy within the first 10 communication rounds and maintains a stable performance plateau thereafter. 
This confirms the effectiveness of personalized initialization in FedEP,  which allows the global model to be better fine-tuned for local distributions, thereby minimizing the number of rounds required to reach optimal performance in federated IoT systems.

\subsection{Personalized Analysis}

\subsubsection{Feature Importance}

\begin{figure*}[t]
    \centering
    \subfigure[TON-IoT]{
        \label{fig:ton_butterfly}
        \hspace{-5mm}
        \includegraphics[width=0.32 \textwidth]{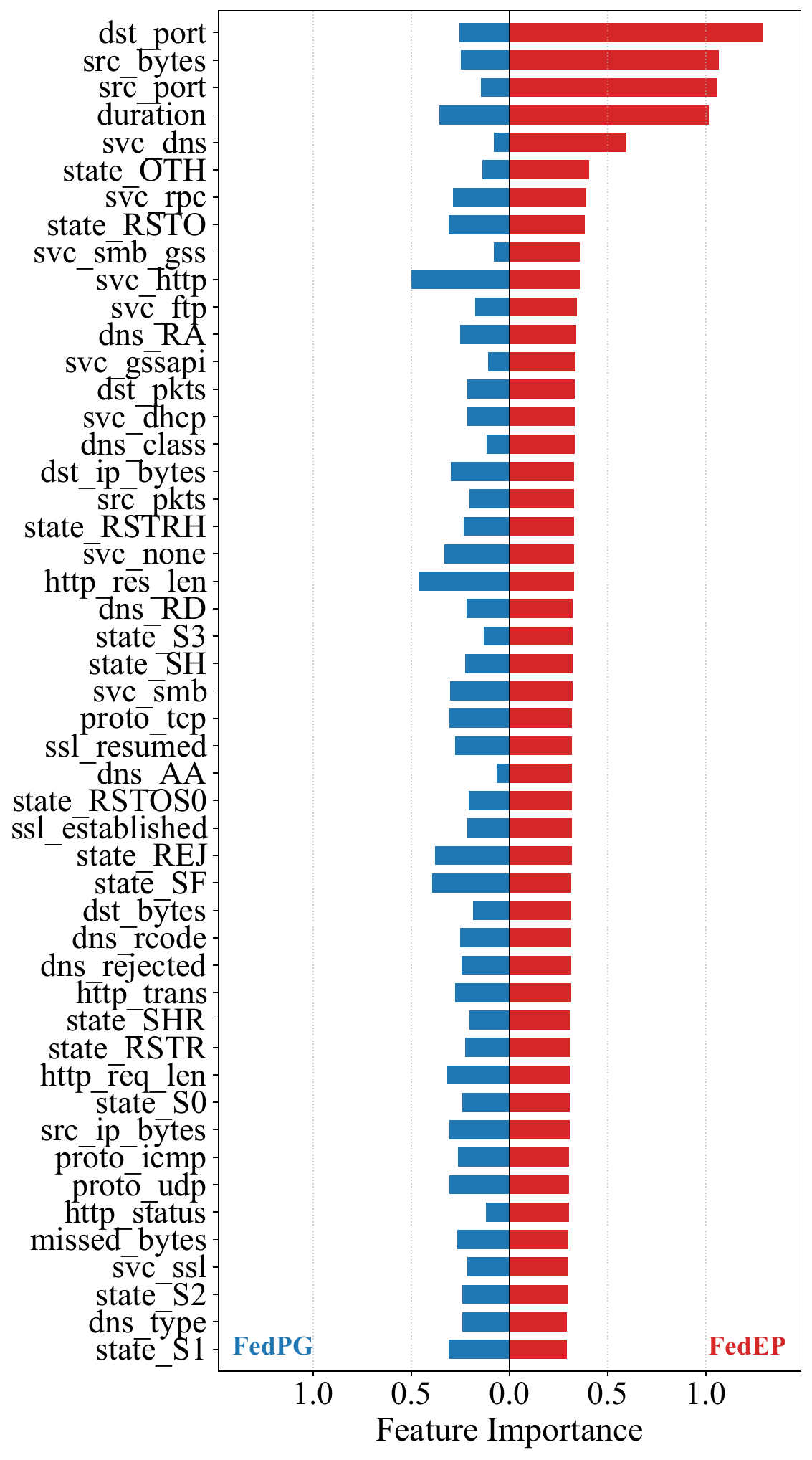}
    }\hfill
    \subfigure[UNSW-NB15]{
        \label{fig:unsw_butterfly}
        \hspace{-8mm}
        \includegraphics[width=0.32 \textwidth]{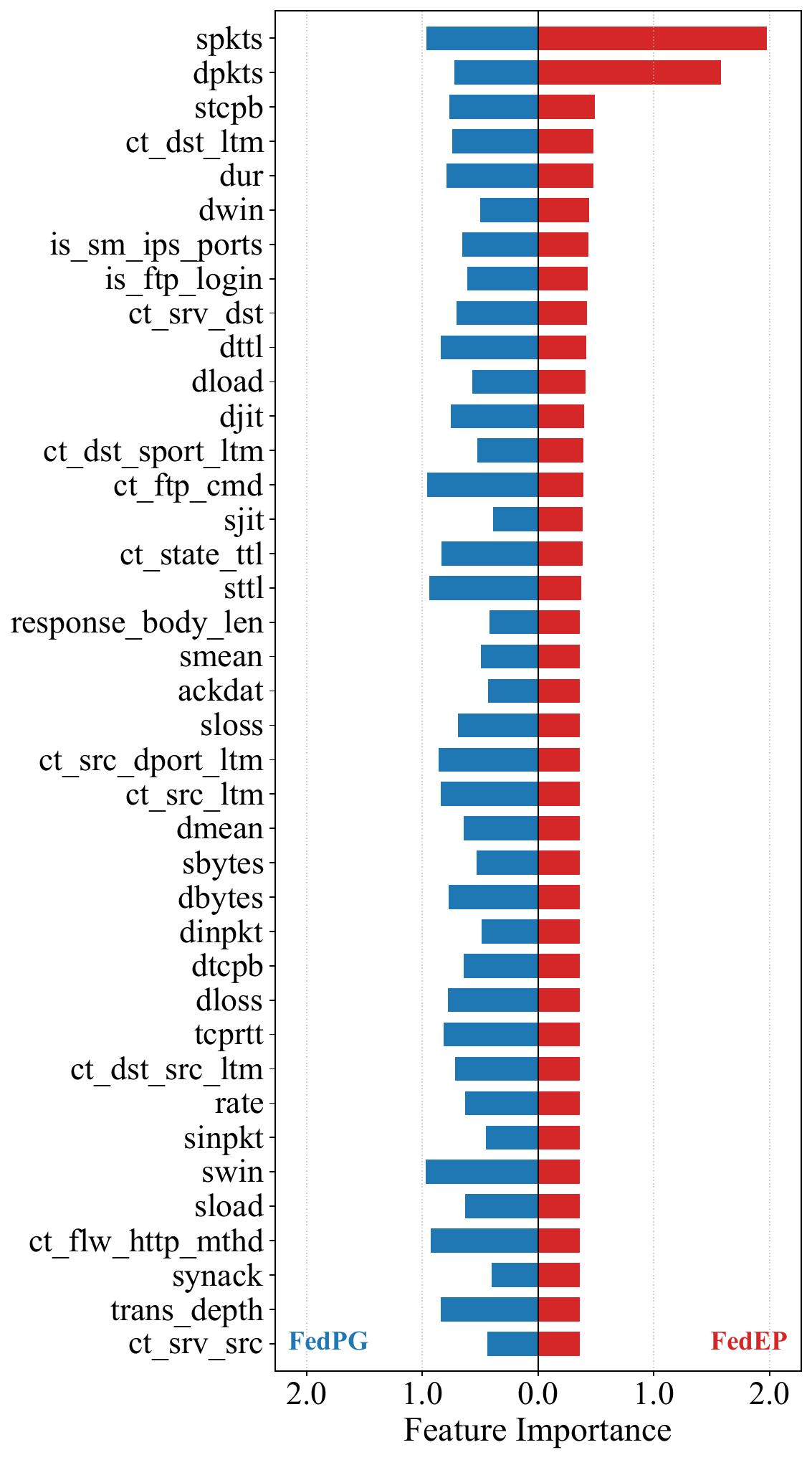}
    }\hfill
    \subfigure[NSL-KDD]{
        \label{fig:kdd_butterfly}
        \hspace{-12mm}
        \includegraphics[width=0.32 \textwidth]{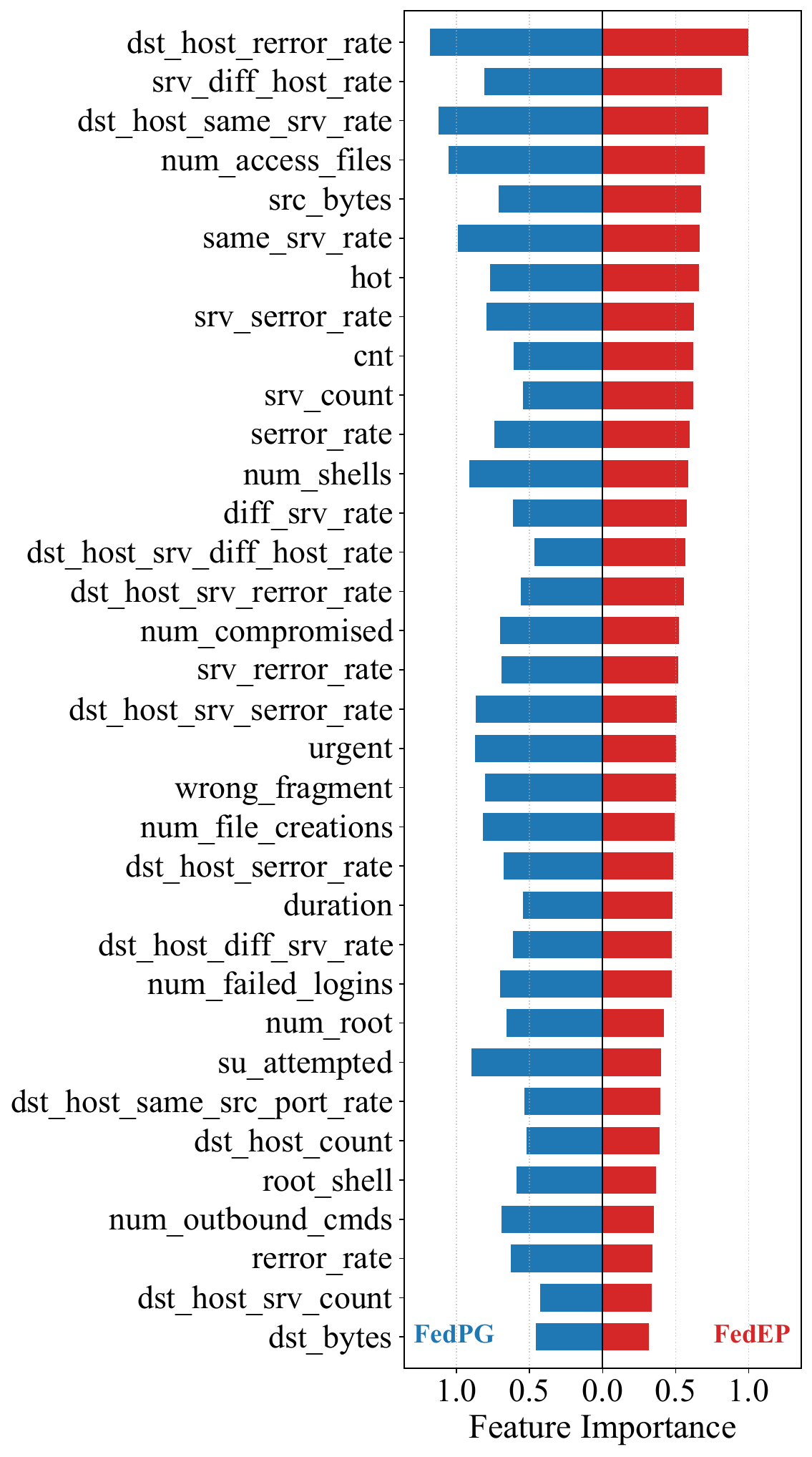}
    }
    
    \caption{Feature importance visualization on (a) TON-IoT, (b) UNSW-NB15, (c) NSL-KDD.}
    \label{fig:three_butterflies}
\end{figure*}

To investigate the impact of the proposed sparse regularization on model interpretability, the feature importance, represented by the sum of absolute weights, is visualized in Fig. \ref{fig:three_butterflies}.
As shown in the red distribution, FedEP exhibits a highly concentrated weight assignment, where only a small subset of key features carries significant weight. In contrast, the baseline FedPG shown in blue assigns non-negligible weights to a much broader range of secondary features. This indicates that FedEP effectively filters out redundant information and suppresses the influence of noise, which is crucial for preventing overfitting in heterogeneous IoT environments.

\subsubsection{Ablation Experiments}


\begin{table}[t]
    \centering
    \caption{Ablation study on TON-IoT.} 
    \label{tab:ablation_ton_left}
    \renewcommand\arraystretch{1.3}
    \begin{tabular}{|W|Z|Z|Z|Z|}
        \hline
        Metrics & Case I & Case II & Case III & Case IV \\
        \hline \hline
        Acc  ($\uparrow$)     & 89.63\% & 89.63\% & 89.71\% & \textbf{90.48\%} \\ \hline
        Pre  ($\uparrow$)     & 91.60\% & 91.60\% & 91.68\% & \textbf{92.48\%} \\ \hline
        Recall ($\uparrow$)   & 96.66\% & 96.66\% & 96.66\% & \textbf{96.67\%} \\ \hline
        FNR    ($\downarrow$)   & 3.34\%  & 3.34\%  & 3.34\%  & \textbf{3.33\%}  \\ \hline
        F1-score ($\uparrow$) & 94.06\% & 94.06\% & 94.11\% & \textbf{94.52\%} \\ \hline
    \end{tabular}
\end{table}

\begin{table}[t]
    \centering
    \caption{Ablation study on UNSW-NB15.} 
    \label{tab:ablation_ton_right}
    \renewcommand\arraystretch{1.3}
    \begin{tabular}{|W|Z|Z|Z|Z|}
        \hline
        Metrics & Case I & Case II & Case III & Case IV \\
        \hline \hline
        Acc    ($\uparrow$)    & 77.84\% & 81.61\% & 75.74\% & \textbf{83.31\%} \\ \hline
        Pre     ($\uparrow$)   & 75.79\% & 79.34\% & 74.09\% & \textbf{82.17\%} \\ \hline
        Recall  ($\uparrow$)   & \textbf{100.00\%} & 99.37\% & \textbf{100.00\%} & 97.00\% \\ \hline
        FNR     ($\downarrow$)   & \textbf{0.00\%}  & 0.63\%  & \textbf{0.00\%}  & 3.00\%  \\ \hline
        F1-score ($\uparrow$) & 86.23\% & 88.23\% & 85.12\% & \textbf{88.97\%} \\ \hline
    \end{tabular}
\end{table}

\begin{table}[t]
    \centering
    \caption{Ablation study on NSL-KDD.} 
    \label{tab:ablation_kdd}
    \renewcommand\arraystretch{1.3}
    \begin{tabular}{|W|Z|Z|Z|Z|}
        \hline
        Metrics  & Case I & Case II & Case III & Case IV \\
        \hline \hline
        Acc    ($\uparrow$)    & 84.09\% & 84.20\% & 84.12\% & \textbf{84.24\%} \\ \hline
        Pre     ($\uparrow$)   & 89.56\% & 89.62\% & \textbf{89.66\%} & \textbf{89.66\%} \\ \hline
        Recall  ($\uparrow$)   & 81.56\% & 81.72\% & 81.50\% & \textbf{81.75\%} \\ \hline
        FNR     ($\downarrow$)   & 18.44\%  & 18.28\%  & 18.50\%  & \textbf{18.25\%}  \\ \hline
        F1-score ($\uparrow$) & 85.38\% & 85.49\% & 85.39\% & \textbf{85.52\%} \\ \hline
    \end{tabular}
\end{table}

\begin{figure*}[t]
	\centering
	
	\subfigure[TON-IoT]{
		\includegraphics[width=0.3\textwidth]{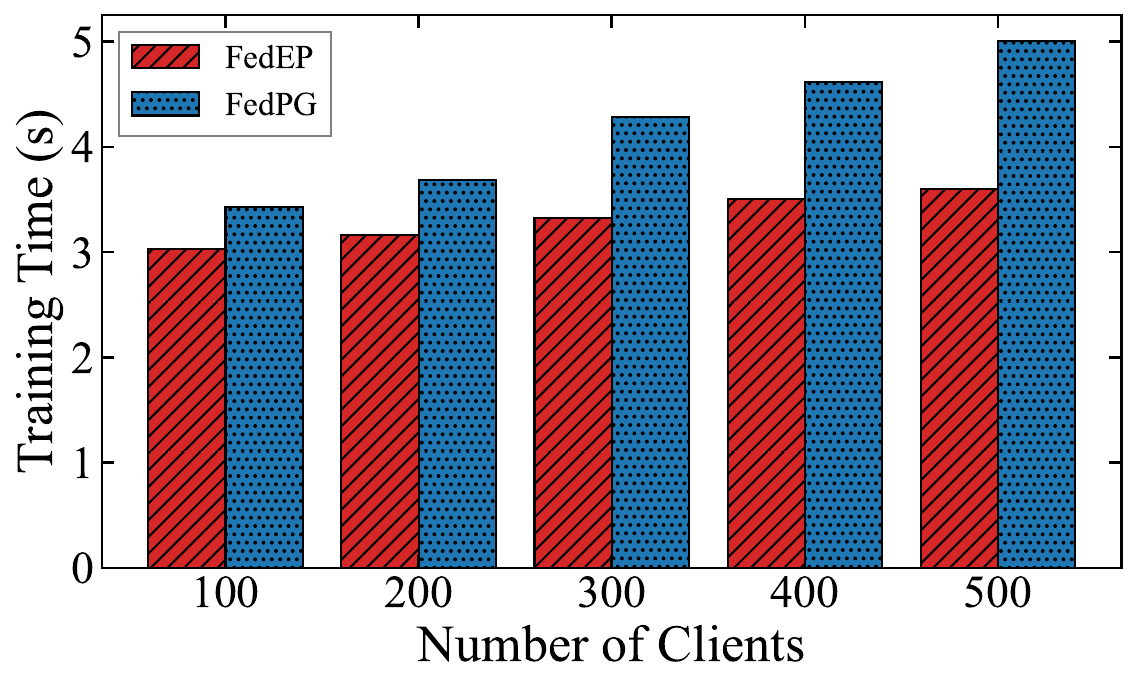}
		\label{fig:ton_method}
	}\hfill
	\subfigure[UNSW-NB15]{
		\includegraphics[width=0.3\textwidth]{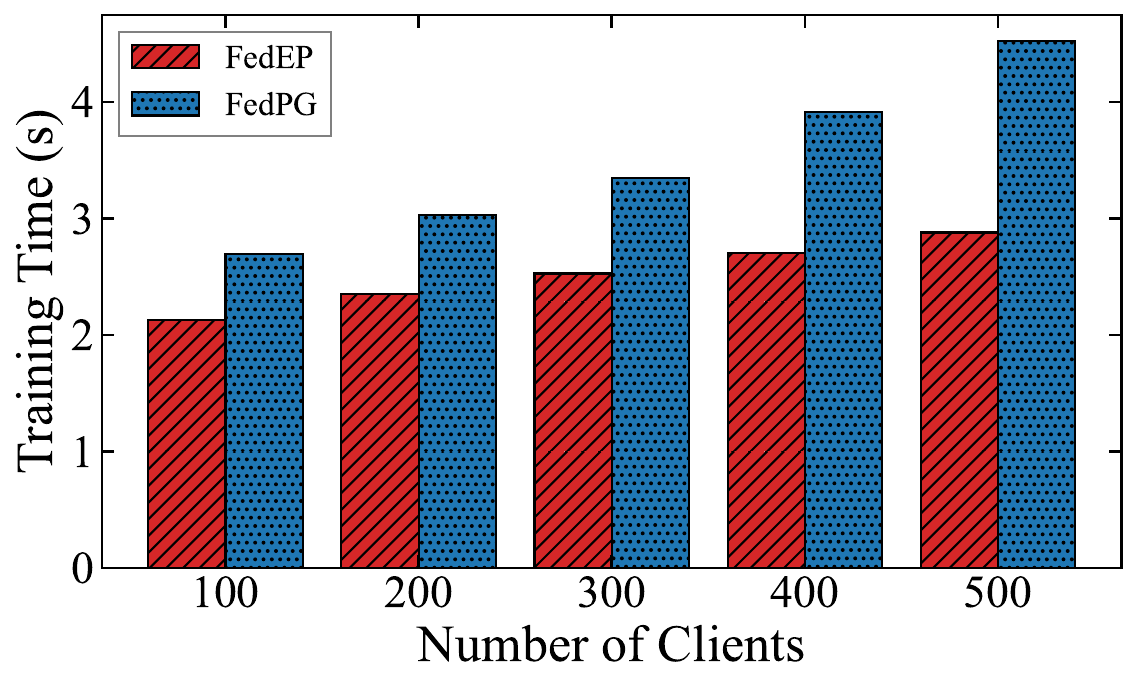}
		\label{fig:unsw_method}
	}\hfill
	\subfigure[NSL-KDD]{
		\includegraphics[width=0.3\textwidth]{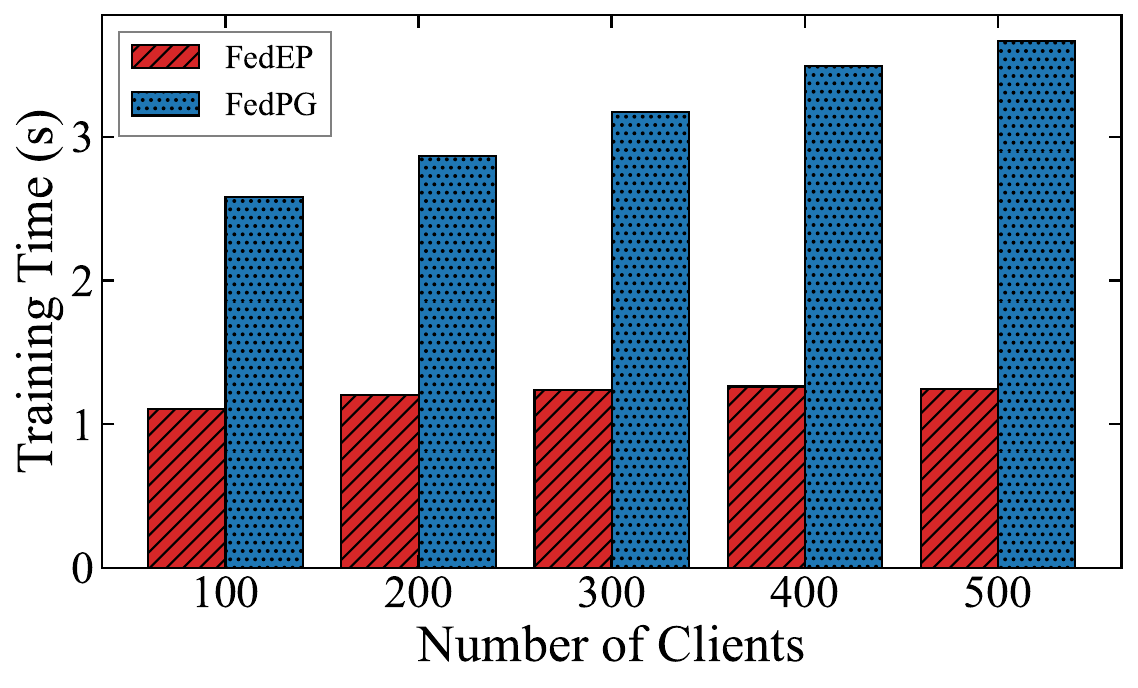}
		\label{fig:kdd_method}
	} 
	\vspace{-2mm}
	\caption{Training time per round as the number of clients scales on (a) TON-IoT, (b) UNSW-NB15, (c) NSL-KDD.}
	\label{fig:overall_results_efficiency1}
\end{figure*}

\begin{figure*}[t]
	\centering
	\subfigure[TON-IoT]{
		\includegraphics[width=0.3\textwidth]{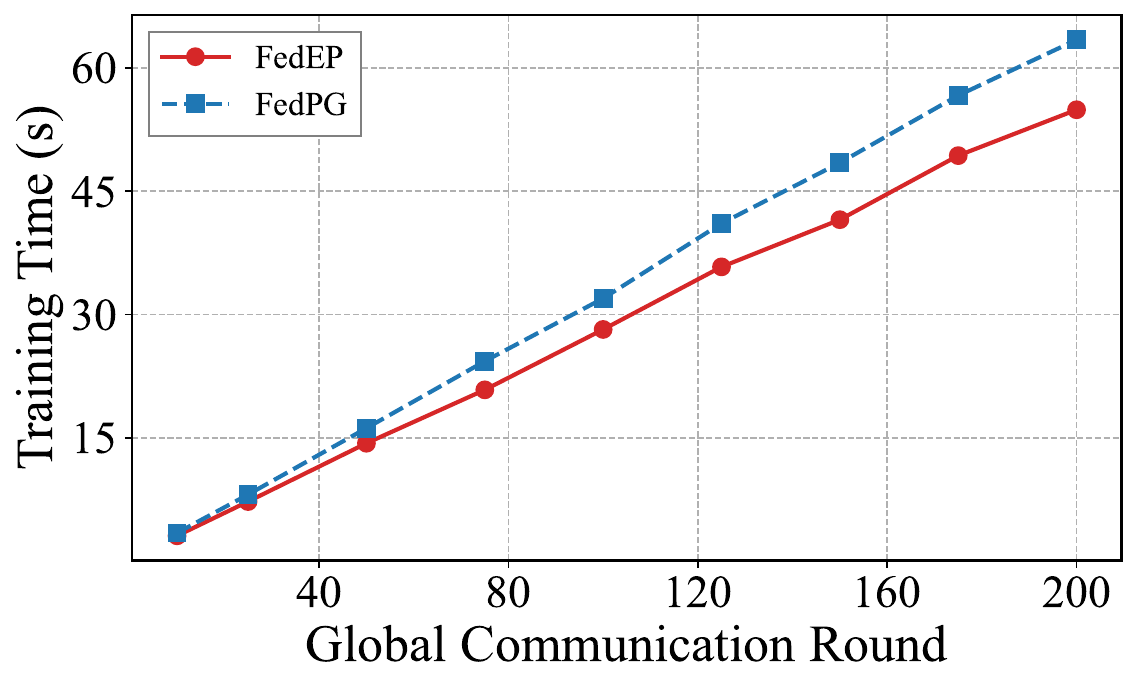}
		\label{fig:ton_time}
	}\hfill
	\subfigure[UNSW-NB15]{
		\includegraphics[width=0.3\textwidth]{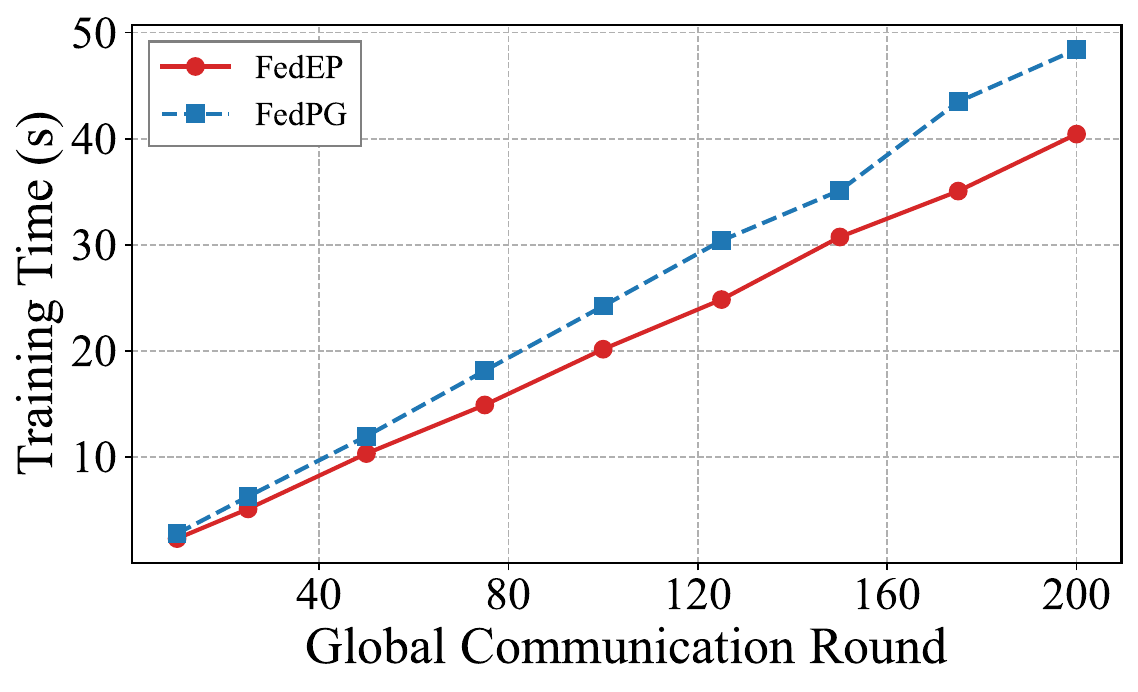}
		\label{fig:unsw_time}
	}\hfill
	\subfigure[NSL-KDD]{
		\includegraphics[width=0.3\textwidth]{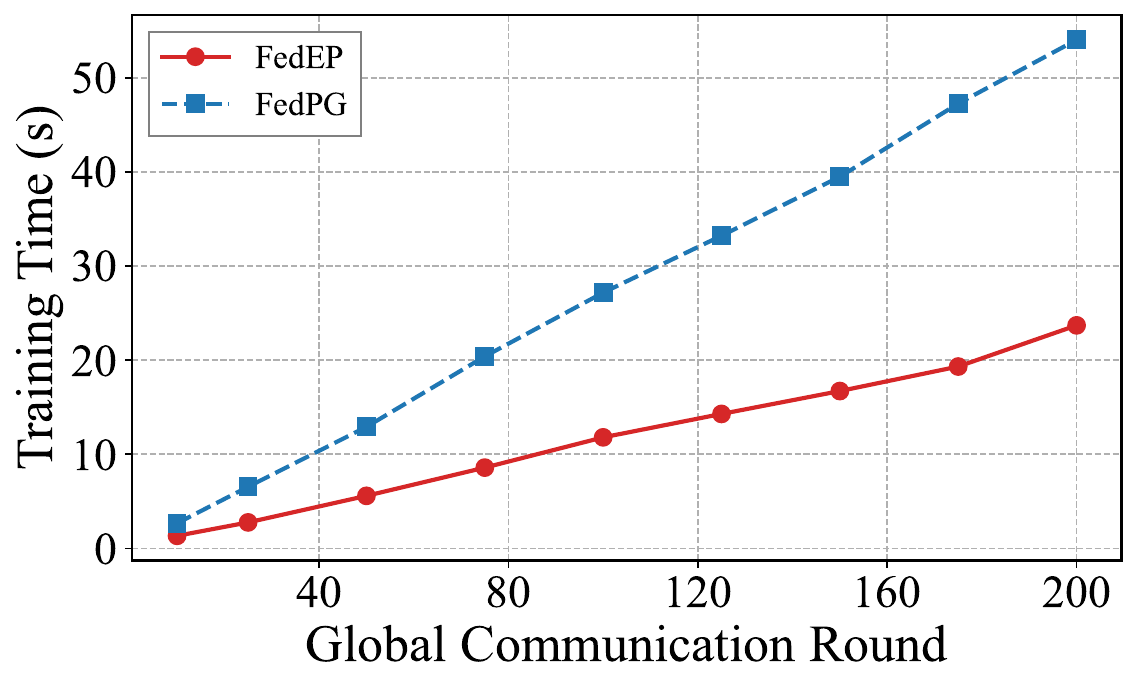}
		\label{fig:kdd_time}
	}
	\vspace{-2mm}
	\caption{Cumulative training time over communication rounds on  (a) TON-IoT, (b) UNSW-NB15, (c) NSL-KDD.}
	\label{fig:overall_results_efficiency2}
\end{figure*}

Ablation studies are conducted across four configurations to verify the individual contributions of each component in the proposed FedEP.

\begin{itemize}
    \item \textbf{Case I:} It removes both the outlier-capture term $S_i$ and the structural sparsity term $\|W_i\|_{2,1}$, i.e.,
    \begin{equation}
        \begin{array}{cl}
            \displaystyle \min_{\{W_i\}, V} & \sum_{i=1}^d \| (I - W_i W_i^{\top}) X_i \|_\textrm{F}^2 \\
            \textrm{s.t.} & W_i = V, \quad W_i^{\top} W_i = I, \quad \forall i \in [d].
        \end{array}
    \end{equation}

    \item \textbf{Case II:} It incorporates the sparse error matrix $S_i$ but omits the row-wise sparsity constraint, i.e.,
    \begin{equation}
        \begin{array}{cl}
            \displaystyle \min_{\{W_i\}, \{S_i\}, V} & \sum_{i=1}^d (\| (I - W_i W_i^{\top})(X_i - S_i) \|_\textrm{F}^2 + \alpha \|S_i\|_1) \\
            \textrm{s.t.} & W_i = V, \quad W_i^{\top} W_i = I, \quad \forall i \in [d].
        \end{array}
    \end{equation}

    \item \textbf{Case III:} It focuses on joint feature selection via the $\ell_{2,1}$-norm, i.e.,
    \begin{equation}
        \begin{array}{cl}
            \displaystyle \min_{\{W_i\}, V}  & \sum_{i=1}^d (\| (I - W_i W_i^{\top}) X_i \|_\textrm{F}^2 + \beta \|W_i\|_{2,1}) \\
            \textrm{s.t.} & W_i = V, \quad W_i^{\top} W_i = I, \quad \forall i \in [d].
        \end{array}
    \end{equation}

    \item \textbf{Case IV:} It is the complete objective function in \eqref{eq:proposed_method}.
\end{itemize}

On TON-IoT, the complete model (Case IV) achieves the highest accuracy, i.e., 90.48\%, and F1-score, i.e., 94.52\%. 
Although the performance gains over Case I - Case III are incremental, the consistent improvement across all metrics indicates that the integration of error decomposition and row-wise sparsity enhances representation stability.
The effectiveness of our proposed FedEP is particularly evident on UNSW-NB15. 
While Cases I and III yield a 100\% recall, their low accuracy and precision suggest a tendency toward over-prediction. 
In contrast, Case IV achieves a more balanced performance with the highest accuracy and F1-score, demonstrating that the $\ell_{2,1}$-norm filters redundant features and the sparse matrix $S_i$ effectively mitigates outlier interference. 
On NSL-KDD, Case IV maintains marginal superiority with 84.24\% accuracy and 85.52\% F1-score. 
The results of all three datasets confirm that the combination of personalization and sparsity enables FedEP to maintain robust detection performance in complex scenarios while processing redundant or corrupted local data.

\subsection{Efficiency Analysis}

\subsubsection{Training Time}

Fig. \ref{fig:overall_results_efficiency1} illustrates the training time per round as the number of clients scales from 100 to 500. It can be seen that FedEP is computationally faster than FedPG, even achieving nearly a 2x improvement on NSL-KDD.
Moreover, Fig. \ref{fig:overall_results_efficiency2} presents a comprehensive comparison of the cumulative training time.
Although both methods exhibit a linear growth relative to the global communication rounds, FedEP demonstrates superior efficiency on the three datasets.
These results indicate that the proposed FedEP effectively mitigates computational overhead across diverse network environments, making it highly suitable for resource-constrained federated IoT environments.

\subsubsection{Impact of Rank Selection}

\begin{table}[t]
    \centering
    \caption{Impact of Rank selection  on TON-IoT.}
    \label{tab:rank_ton}
    \setlength{\tabcolsep}{6pt} 
    \renewcommand\arraystretch{1.3}
    \begin{tabular}{|c|c|c|c|c|}
        \hline
        Rank & \multicolumn{2}{c|}{Acc  ($\uparrow$) } & \multicolumn{2}{c|}{F1-score  ($\uparrow$) } \\
        \cline{2-5}
        (r)  & FedPG & FedEP & FedPG & FedEP \\
        \hline \hline 
        5  & 88.82\% & \textbf{90.48\%} & 93.63\% & \textbf{94.52\%} \\ \hline
        10 & 87.72\% & \textbf{89.69\%} & 93.05\% & \textbf{94.09\%} \\ \hline
        15 & 79.84\% & \textbf{89.74\%} & 87.75\% & \textbf{94.12\%} \\ \hline
        20 & 87.98\% & \textbf{89.92\%} & 93.18\% & \textbf{94.22\%} \\ \hline
        25 & 89.10\% & \textbf{89.38\%} & 93.78\% & \textbf{93.93\%} \\ \hline
        30 & 86.74\% & \textbf{88.42\%} & 92.56\% & \textbf{93.39\%} \\ \hline
    \end{tabular}
\end{table}

\begin{table}[t]
    \centering
    \caption{Impact of Rank selection  on UNSW-NB15.}
    \label{tab:rank_unsw}
    \setlength{\tabcolsep}{6pt} 
    \renewcommand\arraystretch{1.3}
    \begin{tabular}{|c|c|c|c|c|}
        \hline
        Rank & \multicolumn{2}{c|}{Acc ($\uparrow$) } & \multicolumn{2}{c|}{F1-score ($\uparrow$)} \\
        \cline{2-5}
        (r)  & FedPG & FedEP & FedPG & FedEP \\
        \hline \hline
        5  & 80.97\% & \textbf{83.22\%} & 87.24\% & \textbf{89.01\%} \\ \hline
        10 & \textbf{81.19\%} & 77.24\% & \textbf{88.05\%} & 85.91\% \\ \hline
        15 & \textbf{79.85\%} & 72.88\% & \textbf{86.90\%} & 83.65\% \\ \hline
        20 & \textbf{82.04\%} & 81.77\% & 87.15\% & \textbf{88.14\%} \\ \hline
        25 & \textbf{77.81\%} & 76.73\% & 84.54\% & \textbf{85.64\%} \\ \hline
        30 & \textbf{72.94\%} & 70.07\% & 80.39\% & \textbf{82.26\%} \\ \hline
    \end{tabular}
\end{table}

\begin{table}[t]
    \centering
    \caption{Impact of Rank selection on NSL-KDD.}
    \label{tab:rank_kdd}
    \setlength{\tabcolsep}{6pt} 
    \renewcommand\arraystretch{1.3}
    \begin{tabular}{|c|c|c|c|c|}
        \hline
        Rank & \multicolumn{2}{c|}{Acc  ($\uparrow$) } & \multicolumn{2}{c|}{F1-score  ($\uparrow$)} \\
        \cline{2-5}
        (r)  & FedPG & FedEP & FedPG & FedEP \\
        \hline \hline
        5  & 84.09\% & \textbf{84.24\%} & 85.33\% & \textbf{85.52\%} \\ \hline
        10 & 83.42\% & \textbf{83.95\%} & 84.87\% & \textbf{85.23\%} \\ \hline
        15 & \textbf{83.88\%} & 81.23\% & \textbf{85.13\%} & 83.80\% \\ \hline
        20 & \textbf{83.48\%} & 83.15\% & 84.80\% & \textbf{85.05\%} \\ \hline
        25 & 80.99\% & \textbf{84.53\%} & 81.99\% & \textbf{85.97\%} \\ \hline
        30 & 75.62\% & \textbf{83.34\%} & 75.43\% & \textbf{84.44\%} \\ \hline
    \end{tabular}
\end{table}

Table \ref{tab:rank_ton} - Table \ref{tab:rank_kdd} illustrate the impact of rank $r \in [5, 30]$ on detection performance. Across all datasets, FedEP achieves its optimal performance at a low rank ($r=5$), suggesting that the intrinsic structure of IoT traffic is effectively captured by a compact subspace. 
As $r$ increases, both methods experience performance degradation, as a higher-dimensional subspace tends to overfit noise and outliers, thereby blurring the boundary between normal and anomalous samples. However, FedEP demonstrates superior structural stability compared to FedPG. For instance, while FedPG's accuracy on NSL-KDD plummets to 75.62\% at $r=30$, FedEP maintains a significantly more robust 83.34\%. 
Even under suboptimal rank selections, FedEP retains higher F1-scores, as the sparse component $S_i$ effectively absorbs residual noise and prevents it from distorting the estimated subspace. In summary, FedEP exhibits greater resilience to over-parameterization, validating its efficacy in modeling complex federated data.

\section{Conclusion}\label{conclusions}

In this paper, we propose an efficient personalized federated PCA (FedEP) framework for anomaly detection in IoT networks. Unlike existing federated PCA methods that enforce strict consensus constraints, FedEP enables local gateways to maintain personalized model parameters while benefiting from global knowledge sharing. The integration of robust sparse regularization with $\ell_{2,1}$-norm and $\ell_1$-norm enhances the model's ability to handle noise and outliers in heterogeneous IoT environments. 
Furthermore, we develop an efficient manifold optimization algorithm based on ADMM with theoretical convergence guarantees. 
Numerical experiments verify that our proposed FedEP  achieves better accuracy and higher computational efficiency compared to the benchmark FedPG.

In the future, we are interested in extending FedEP to nonlinear models using kernel-based techniques or deep learning. Additionally,  investigating the compressibility of personalized models to reduce communication overhead is crucial for scalable deployment on resource-constrained IoT devices.

\bibliographystyle{IEEEtran}
\bibliography{mybibfile}

\begin{thebibliography}{10}
\providecommand{\url}[1]{#1}
\csname url@samestyle\endcsname
\providecommand{\newblock}{\relax}
\providecommand{\bibinfo}[2]{#2}
\providecommand{\BIBentrySTDinterwordspacing}{\spaceskip=0pt\relax}
\providecommand{\BIBentryALTinterwordstretchfactor}{4}
\providecommand{\BIBentryALTinterwordspacing}{\spaceskip=\fontdimen2\font plus
\BIBentryALTinterwordstretchfactor\fontdimen3\font minus
  \fontdimen4\font\relax}
\providecommand{\BIBforeignlanguage}[2]{{%
\expandafter\ifx\csname l@#1\endcsname\relax
\typeout{** WARNING: IEEEtran.bst: No hyphenation pattern has been}%
\typeout{** loaded for the language `#1'. Using the pattern for}%
\typeout{** the default language instead.}%
\else
\language=\csname l@#1\endcsname
\fi
#2}}
\providecommand{\BIBdecl}{\relax}
\BIBdecl

\bibitem{aouedi2024survey}
O.~Aouedi, T.-H. Vu, A.~Sacco, D.~C. Nguyen, K.~Piamrat, G.~Marchetto, and
  Q.-V. Pham, ``A survey on intelligent {I}nternet of things: Applications,
  security, privacy, and future directions,'' \emph{IEEE Communications Surveys
  \& Tutorials}, vol.~27, no.~2, pp. 1238--1292, 2025.

\bibitem{rejeb2022big}
A.~Rejeb, K.~Rejeb, S.~Simske, H.~Treiblmaier, and S.~Zailani, ``The big
  picture on the {I}nternet of things and the smart city: A review of what we
  know and what we need to know,'' \emph{Internet of Things}, vol.~19, p.
  100565, 2022.

\bibitem{zhang2025internet}
L.~Zhang, I.~K. Dabipi, and W.~L. Brown~Jr, ``Internet of things applications
  for agriculture,'' \emph{Internet of Things A to Z: Technologies and
  Applications}, pp. 327--343, 2025.

\bibitem{bollineni2025iot}
C.~Bollineni, M.~Sharma, A.~Hazra, P.~Kumari, S.~Manipriya, and A.~Tomar, ``Iot
  for next-generation smart healthcare: A comprehensive survey,'' \emph{IEEE
  Internet of Things Journal}, vol.~12, no.~16, pp. 32\,616--32\,639, 2025.

\bibitem{yalli2025systematic}
J.~S. Yalli, M.~H. Hasan, L.~T. Jung, A.~I. Yerima, D.~A. Aliyu, U.~D. Maiwada,
  S.~M. Al-Selwi, and M.~U. Shaikh, ``A systematic review for evaluating {IoT}
  security: A focus on authentication, protocols and enabling technologies,''
  \emph{IEEE Internet of Things Journal}, vol.~12, no.~11, pp.
  18\,908--18\,928, 2025.

\bibitem{tu2025distributed}
J.~Tu, L.~Yang, and J.~Cao, ``Distributed machine learning in edge computing:
  Challenges, solutions and future directions,'' \emph{ACM Computing Surveys},
  vol.~57, no.~5, pp. 1--37, 2025.

\bibitem{huang2025deep}
H.~Huang, P.~Wang, J.~Pei, J.~Wang, S.~Alexanian, and D.~Niyato, ``Deep
  learning advancements in anomaly detection: A comprehensive survey,''
  \emph{IEEE Internet of Things Journal}, vol.~12, no.~21, pp.
  44\,318--44\,342, 2025.

\bibitem{inuwa2024comparative}
M.~M. Inuwa and R.~Das, ``A comparative analysis of various machine learning
  methods for anomaly detection in cyber attacks on {IoT} networks,''
  \emph{Internet of Things}, vol.~26, p. 101162, 2024.

\bibitem{adhikari2024recent}
D.~Adhikari, W.~Jiang, J.~Zhan, D.~B. Rawat, and A.~Bhattarai, ``Recent
  advances in anomaly detection in internet of things: Status, challenges, and
  perspectives,'' \emph{Computer Science Review}, vol.~54, p. 100665, 2024.

\bibitem{zhou2017security}
J.~Zhou, Z.~Cao, X.~Dong, and A.~V. Vasilakos, ``Security and privacy for
  cloud-based {IoT}: Challenges,'' \emph{IEEE Communications Magazine},
  vol.~55, no.~1, pp. 26--33, 2017.

\bibitem{trilles2024anomaly}
S.~Trilles, S.~S. Hammad, and D.~Iskandaryan, ``Anomaly detection based on
  artificial intelligence of things: A systematic literature mapping,''
  \emph{Internet of Things}, vol.~25, p. 101063, 2024.

\bibitem{kounoudes2020mapping}
A.~D. Kounoudes and G.~M. Kapitsaki, ``A mapping of {IoT} user-centric privacy
  preserving approaches to the {GDPR},'' \emph{Internet of Things}, vol.~11, p.
  100179, 2020.

\bibitem{ieracitano2020novel}
C.~Ieracitano, A.~Adeel, F.~C. Morabito, and A.~Hussain, ``A novel statistical
  analysis and autoencoder driven intelligent intrusion detection approach,''
  \emph{Neurocomputing}, vol. 387, pp. 51--62, 2020.

\bibitem{zulfiqar2024deepdetect}
Z.~Zulfiqar, S.~U. Malik, S.~A. Moqurrab, Z.~Zulfiqar, U.~Yaseen, and
  G.~Srivastava, ``Deep{D}etect: An innovative hybrid deep learning framework
  for anomaly detection in {IoT} networks,'' \emph{Journal of Computational
  Science}, vol.~83, p. 102426, 2024.

\bibitem{mcmahan2017communication}
B.~McMahan, E.~Moore, D.~Ramage, S.~Hampson, and B.~A. y~Arcas,
  ``Communication-efficient learning of deep networks from decentralized
  data,'' in \emph{Artificial Intelligence and Statistics}.\hskip 1em plus
  0.5em minus 0.4em\relax PMLR, 2017, pp. 1273--1282.

\bibitem{nguyen2021federated}
D.~C. Nguyen, M.~Ding, P.~N. Pathirana, A.~Seneviratne, J.~Li, and H.~V. Poor,
  ``Federated learning for {I}nternet of things: A comprehensive survey,''
  \emph{IEEE Communications Surveys \& Tutorials}, vol.~23, no.~3, pp.
  1622--1658, 2021.

\bibitem{zhang2024privacy}
Y.~Zhang, B.~Suleiman, M.~J. Alibasa, and F.~Farid, ``Privacy-aware anomaly
  detection in {IoT} environments using {FedGroup}: A group-based federated
  learning approach,'' \emph{Journal of Network and Systems Management},
  vol.~32, no.~1, p.~20, 2024.

\bibitem{tang2022computational}
S.~Tang, L.~Chen, K.~He, J.~Xia, L.~Fan, and A.~Nallanathan, ``Computational
  intelligence and deep learning for next-generation edge-enabled industrial
  {IoT},'' \emph{IEEE Transactions on Network Science and Engineering},
  vol.~10, no.~5, pp. 2881--2893, 2023.

\bibitem{han2019federated}
Y.~Han, D.~Li, H.~Qi, J.~Ren, and X.~Wang, ``Federated learning-based
  computation offloading optimization in edge computing-supported {I}nternet of
  things,'' in \emph{Proceedings of the ACM Turing Celebration
  Conference-China}, 2019, pp. 1--5.

\bibitem{zhou2024reconstructed}
X.~Zhou, J.~Wu, W.~Liang, K.~I.-K. Wang, Z.~Yan, L.~T. Yang, and Q.~Jin,
  ``Reconstructed graph neural network with knowledge distillation for
  lightweight anomaly detection,'' \emph{IEEE Transactions on Neural Networks
  and Learning Systems}, vol.~35, no.~9, pp. 11\,817--11\,828, 2024.

\bibitem{carter2022fast}
J.~Carter, S.~Mancoridis, and E.~Galinkin, ``Fast, lightweight {IoT} anomaly
  detection using feature pruning and {PCA},'' in \emph{Proceedings of the 37th
  ACM/SIGAPP Symposium on Applied Computing}, 2022, pp. 133--138.

\bibitem{xiu2025bi}
X.~Xiu, C.~Huang, P.~Shang, and W.~Liu, ``Bi-sparse unsupervised feature
  selection,'' \emph{IEEE Transactions on Image Processing}, vol.~34, pp.
  7407--7421, 2025.

\bibitem{grammenos2020federated}
A.~Grammenos, R.~Mendoza~Smith, J.~Crowcroft, and C.~Mascolo, ``Federated
  principal component analysis,'' \emph{Advances in Neural Information
  Processing Systems}, vol.~33, pp. 6453--6464, 2020.

\bibitem{nguyen2024federated}
T.-A. Nguyen, L.~T. Le, T.~D. Nguyen, W.~Bao, S.~Seneviratne, C.~S. Hong, and
  N.~H. Tran, ``Federated {PCA} on {G}rassmann manifold for {IoT} anomaly
  detection,'' \emph{IEEE/ACM Transactions on Networking}, vol.~32, no.~5, pp.
  4456--4471, 2024.

\bibitem{wang2023high}
K.~Wang and Z.~Song, ``High-dimensional cross-plant process monitoring with
  data privacy: A federated hierarchical sparse {PCA} approach,'' \emph{IEEE
  Transactions on Industrial Informatics}, vol.~20, no.~3, pp. 4385--4396,
  2024.

\bibitem{luo2024privacy}
G.~Luo, N.~Chen, J.~He, B.~Jin, Z.~Zhang, and Y.~Li, ``Privacy-preserving
  clustering federated learning for {non-IID} data,'' \emph{Future Generation
  Computer Systems}, vol. 154, pp. 384--395, 2024.

\bibitem{hoang2018pca}
D.~H. Hoang and H.~D. Nguyen, ``A {PCA}-based method for {IoT} network traffic
  anomaly detection,'' in \emph{2018 20th International Conference on Advanced
  Communication Technology (ICACT)}.\hskip 1em plus 0.5em minus 0.4em\relax
  IEEE, 2018, pp. 381--386.

\bibitem{chen2021bridging}
Y.~Chen, J.~Fan, C.~Ma, and Y.~Yan, ``Bridging convex and nonconvex
  optimization in robust {PCA}: Noise, outliers, and missing data,''
  \emph{Annals of Statistics}, vol.~49, no.~5, p. 2948, 2021.

\bibitem{rousseeuw2018anomaly}
P.~J. Rousseeuw and M.~Hubert, ``Anomaly detection by robust statistics,''
  \emph{Wiley Interdisciplinary Reviews: Data Mining and Knowledge Discovery},
  vol.~8, no.~2, p. e1236, 2018.

\bibitem{liu2025similarity}
Z.~Liu, H.~Lv, X.~Liu, C.~Ma, F.~Wu, L.~Liu, and L.~Cui, ``Similarity and
  diversity: {PCA}-based contribution evaluation in federated learning,''
  \emph{IEEE Internet of Things Journal}, vol.~12, no.~12, pp.
  20\,393--20\,405, 2025.

\bibitem{sun2023learning}
J.~Sun, X.~Xiu, Z.~Luo, and W.~Liu, ``Learning high-order multi-view
  representation by new tensor canonical correlation analysis,'' \emph{IEEE
  Transactions on Circuits and Systems for Video Technology}, vol.~33, no.~10,
  pp. 5645--5654, 2023.

\bibitem{liu2024towards}
J.~Liu, M.~Feng, X.~Xiu, and W.~Liu, ``Towards robust and sparse linear
  discriminant analysis for image classification,'' \emph{Pattern Recognition},
  vol. 153, p. 110512, 2024.

\bibitem{zou2006sparse}
H.~Zou, T.~Hastie, and R.~Tibshirani, ``Sparse principal component analysis,''
  \emph{Journal of Computational and Graphical Statistics}, vol.~15, no.~2, pp.
  265--286, 2006.

\bibitem{chen2020alternating}
S.~Chen, S.~Ma, L.~Xue, and H.~Zou, ``An alternating manifold proximal gradient
  method for sparse principal component analysis and sparse canonical
  correlation analysis,'' \emph{INFORMS Journal on Optimization}, vol.~2,
  no.~3, pp. 192--208, 2020.

\bibitem{zhu2025sparse}
Y.~Zhu, W.~Liu, X.~Xiu, and J.~Sun, ``Sparse tensor {CCA} via manifold
  optimization for multi-view learning,'' \emph{IEEE Transactions on Circuits
  and Systems for Video Technology}, 2025.

\bibitem{xiao2018regularized}
X.~Xiao, Y.~Li, Z.~Wen, and L.~Zhang, ``A regularized semi-smooth {N}ewton
  method with projection steps for composite convex programs,'' \emph{Journal
  of Scientific Computing}, vol.~76, no.~1, pp. 364--389, 2018.

\bibitem{donoho1995noising}
D.~L. Donoho, ``De-noising by soft-thresholding,'' \emph{IEEE Transactions on
  Information Theory}, vol.~41, no.~3, pp. 613--627, 1995.

\bibitem{hager1989updating}
W.~W. Hager, ``Updating the inverse of a matrix,'' \emph{SIAM Review}, vol.~31,
  no.~2, pp. 221--239, 1989.

\bibitem{boyd2011distributed}
S.~Boyd, N.~Parikh, E.~Chu, B.~Peleato, J.~Eckstein \emph{et~al.},
  ``Distributed optimization and statistical learning via the alternating
  direction method of multipliers,'' \emph{Foundations and
  Trends{\textregistered} in Machine learning}, vol.~3, no.~1, pp. 1--122,
  2011.

\bibitem{li2025riemannian}
J.~Li, S.~Ma, and T.~Srivastava, ``A {R}iemannian alternating direction method
  of multipliers,'' \emph{Mathematics of Operations Research}, vol.~50, no.~4,
  pp. 3222--3242, 2025.

\bibitem{booij2021ton_iot}
T.~M. Booij, I.~Chiscop, E.~Meeuwissen, N.~Moustafa, and F.~T. Den~Hartog,
  ``{ToN}\_{IoT}: The role of heterogeneity and the need for standardization of
  features and attack types in {IoT} network intrusion data sets,'' \emph{IEEE
  Internet of Things Journal}, vol.~9, no.~1, pp. 485--496, 2021.

\bibitem{moustafa2015unsw}
N.~Moustafa and J.~Slay, ``{UNSW-NB15}: A comprehensive data set for network
  intrusion detection systems ({UNSW-NB15} network data set),'' in \emph{2015
  Military Communications and Information Systems Conference (MilCIS)}.\hskip
  1em plus 0.5em minus 0.4em\relax IEEE, 2015, pp. 1--6.

\bibitem{tavallaee2009detailed}
M.~Tavallaee, E.~Bagheri, W.~Lu, and A.~A. Ghorbani, ``A detailed analysis of
  the {KDD} {CUP} 99 data set,'' in \emph{2009 IEEE Symposium on Computational
  Intelligence for Security and Defense Applications}.\hskip 1em plus 0.5em
  minus 0.4em\relax IEEE, 2009, pp. 1--6.

\end{thebibliography}

\end{document}